\documentclass{article}

\usepackage{arxiv}

\usepackage[utf8]{inputenc} 
\usepackage[T1]{fontenc}    
\usepackage{hyperref}       
\usepackage{url}            
\usepackage{booktabs}       
\usepackage{amsfonts}       
\usepackage{nicefrac}       
\usepackage{microtype}      
\usepackage{lipsum}		    
\usepackage{graphicx}
\usepackage{natbib}
\usepackage{doi}
\usepackage{subcaption}
\graphicspath{ {./images/} }
\usepackage{enumitem}

\usepackage{amsmath,amssymb}

\usepackage{multirow}
\usepackage{xcolor}
\usepackage{adjustbox}
\usepackage{longtable}
\usepackage{pythonhighlight}
\usepackage{lscape}
\usepackage{algorithm}
\usepackage{algpseudocode}

\newcommand{\M}{\mathcal{M}}

\newcommand{\R}{\mathbb{R}}

\newcommand{\wfuzzy}{\left(w^L_{j}, w^R_{j}, \xi^L_{j}, \xi^R_{j}\right)_{L^{w}_{j}R^{w}_{j}}}

\newcommand{\fuzzyzero}{\tilde{\boldsymbol{0}}}
\newcommand{\fuzzyone}{\tilde{\boldsymbol{1}}}
\newcommand{\xmark}{\ding{55}}
\newcommand{\sign}{\hbox{sign}}

\usepackage{amsmath}
\usepackage{enumitem}
\usepackage{multirow}
\usepackage{xcolor}
\usepackage{adjustbox}
\usepackage{longtable}
\usepackage{pythonhighlight}
\usepackage{amsthm}

\usepackage{tabularx}
\usepackage{pifont}
\usepackage{multirow}
\usepackage{mathrsfs}

\usepackage{tikz}
\usepackage{pgfplots}
\pgfplotsset{width=\linewidth, compat=1.18}

\newtheorem{theorem}{Theorem}
\newtheorem{definition}{Definition}
\newtheorem{proposition}{Proposition}
\newtheorem{lemma}{Lemma}
\newtheorem{example}{Example}

\usepackage{algorithm}
\usepackage{algpseudocode}

\usepackage{amsthm}

\title{Unweighted ranking for value-based decision making with uncertainty}

\author{
\href{https://orcid.org/0000-0001-8332-0381}    {\includegraphics[scale=0.1]{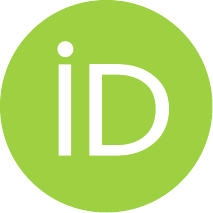}\hspace{1mm}
    Aar{\'o}n~L{\'o}pez-Garc{\'i}a}\\
	Departament de Matem{\'a}tiques per a l'Economia i l'Empresa\\ Facultat d'Economia\\
    Universitat de Val{\'e}ncia\\
	\texttt{aaron.lopez@uv.es}
	\And
\href{https://orcid.org/0000-0003-3728-1785}    {\includegraphics[scale=0.1]{orcid.pdf}\hspace{1mm}
    Natalia Criado}\\
	Valencian Research Institute for Artificial Intelligence\\ Universitat Polit{\`e}cnica de Val{\`e}ncia\\
	\texttt{ncriado@upv.es}\\
    \And
\href{https://orcid.org/0000-0002-6041-178X}    {\includegraphics[scale=0.1]{orcid.pdf}\hspace{1mm}
    Jose Such}\\
	Valencian Research Institute for Artificial Intelligence\\ Universitat Polit{\`e}cnica de Val{\`e}ncia\\
	\texttt{jsuch@upv.es}\\
}

\renewcommand{\shorttitle}{Unweighted ranking for value-based decision making with uncertainty}

\hypersetup{
    pdftitle={Unweighted ranking for value-based decision making with uncertainty},
    pdfsubject={cs.AI},
    pdfauthor={Aaron~Lopez-Garcia},
}

\begin{document}
\maketitle

\begin{abstract}
    As intelligent systems are increasingly implemented in our society to make autonomous decisions, their commitment to human values raises serious concerns. Their alignment with human values remains a critical challenge because it can jeopardise the integrity and security of citizens. For this reason, an innovative human-centred and values-driven approach to decision making is required. In this work, we introduce the Fuzzy-Unweighted Value-Based Decision Making (FUW-VBDM) framework, where agents incorporate both quantitative and qualitative criteria to generate human-centred decisions. We also address the normative bias introduced by stakeholders with arbitrary weights by removing prior weights and introducing a fuzzy domain of decision variables defined for a score function. This concept allows us to generalise any VBDM problem as the search for feasible solutions when optimising the score in the weight domain. To provide a solution to FUW-VBDM, we present Rankzzy, a customizable unweighted ranking method that integrates fuzzy-based reasoning to quantify uncertainty. We mathematically prove the consistency of the Rankzzy for any admissible configuration selected by stakeholders. We show the applicability of our method through an illustrative case study, which we also use as a running example. The evaluation conducted indicates a reduced computational cost in large-scale value-based decision-making problems and a strong rank performance regarding existing approaches when employing the aggregation via Pythagorean means.
\end{abstract}


\section{Introduction}\label{s:intro}

The incorporation of AI-based multi-agent systems into critical infrastructure to solve the challenges of advanced societies is already a fact. Its application and usefulness are indisputable, generating an unprecedented impact on the global world. However, its impact on society also raises ethical and moral concerns. This issue has been extensively addressed by key decision-makers in prestigious organisations such as the \cite{EU_AI_Guidelines} and the \cite{IEEE_AS}. Their discourse is increasingly focused on the commitment to accountability, fairness and transparency to ensure the safety-first principle of citizens. As AI systems play a key role in everyday decisions, their moral implications should not be overlooked \cite{LoPiano2020}. Hence, decision making requires a human-centred vision and a value-aligned commitment to be accepted by society without taking collateral risks.

Value-based decision-making (VBDM) is a very promising approach to address this conflictive scenario. Although it can be understood as a very complex task \cite{Ford1994,wittmer2019}, it is absolutely necessary in our society. When decision-makers plan to execute an action, their expert judgment has to be not only aligned with the application itself but also with the impact that such a decision will have on our society. The main problem does not only lie in assessing human values that affect the decision, but in the conflict of interests that arises when combining the knowledge-based and value-based criteria. For instance, a business plan focusing on social choice has to assess corporate actions considering the impact that such operations will have on both society and the company (for case studies see \cite{Antunes2006} or \cite{Christen2012}). The use of ethical and moral judgment within the decision process makes it even more complicated because there is no straightforward definition when referring to morality and ethics \cite{MacAskill2020, Astobiza2021, Beckstead2023}. Thus, multi-agent systems (\textbf{MAS}) must deal with the trade-offs between domain criteria and the ethical implications that their selection will entail for other alternatives \cite{Lera2022, LeraLeri2024}. To this end, there is a need to combine attributes of a \textbf{quantitative} and \textbf{qualitative} nature in decision models to ensure the proper integration of VBDM.

The problem of making decisions considering human values has been vastly studied and approached from many perspectives. A very successful solution for this problem is the use of ranking systems since they ascertain the preferences among the actions to perform \cite{Roth1992}. In this manner, we can establish a hierarchy by aggregating \cite{Barbera2004} or mapping \cite{Moretti2017} them to obtain individuals' priorities. However, in most cases, the approaches for addressing value-alignment problems lack consistency and scalability because they are designed for computing well-defined data structures. Moreover, the information and knowledge about the problem are usually attached to the presence of bias, uncertainty and non-objective criteria \cite{Reyna2009}.

On the one hand, the use of ethical and moral judgment within the decision process constitutes a serious concern \cite{MacAskill2020}. There is no straightforward definition when referring to morality and ethics \cite{Astobiza2021, Beckstead2023}. Vagueness is present in this decision scenario, and so we must account for it. For VBDM problems, not only are ethical perception and domain criteria subject to \textbf{uncertainty}, but also the environment and problem representation \cite{Payzan2013}, known as second-order uncertainty \cite{Bach2011}. The underlying uncertainty attached to values requires the use of fuzzy logic to model their degree of truth \cite{Zadeh1965, Ramik1985}. As such, subjectivity and non-accurate attributes can be described by means of fuzzy numbers \cite{Bojadziev1996} as well as the ambiguity of moral assessments  \cite{Kleijnen2001, Walstrom2006} to establish an \textbf{order} relationship. There exists a vast literature focused on the implementation of fuzzy logic in decision-making problems (see, for instance, \cite{Klir1988, Mardani2015, Kahraman2015, Xu2016}) and also on the study of uncertainty sources that may arise when applying decision aid techniques \cite{French1995}.

On the other hand, it is well-known that conflicts of interest are a significant concern in decision-making \cite{LopezGarcia2023}. How values are aggregated in a utility function can lead to biased systems \cite{Serramia2020}, thus compromising the credibility of the result obtained. A fundamental component is the weighting schemes. Their impact in decision-making is a determining factor when designing intelligent MAS since they convey partial relevance of the values \cite{Roy1996}. One of the widely assumed assumptions in this type of problem is that the weights can be set a priori as a fixed vector \cite{Triantaphyllou2000, Doumpos2013}. However, it has already been shown that this implies a direct bias that should not be neglected \cite{Jacquet1982, Liern2020}. To handle the inherent difficulty of weight elicitation, some authors have addressed this task through an \textbf{unweighed} approach that utilises decision intervals instead of single schemes \cite{Liern2020, LopezGarcia2023}. Thus, the result can be understood as the range of scores linked to each action. This offers us a more complete and informative overview of the actions and their alignment with the value system considered, which also eliminates the prior bias of fixed decision weights.

Considering all these gaps previously described, there is a need for a framework that supports the decision-making stage in determining the most appropriate solution that aligns with human values. Since the existing methods struggle to combine the quantitative and qualitative nature of values as well as to inherently approach the implicit uncertainty and biases, an end-to-end system that effectively integrates the existing limitations is required. For this purpose, this paper introduces the Fuzzy Unweighted Value-Based Decision-Making framework for dealing with VBDM problems utilising a MAS perspective. It is defined as a tuple composed of the groups of interests - i.e. the sets of agents, actions and values involved in the decision problem - and the definition of the mathematical concepts that lead to the outcome. In particular, a fuzzy order is defined to compare values, the formalisation of a decision matrix with quantitative and/or qualitative variables, the domain of weights to assess values and perform an unweighted approach, and a flexible aggregation function customizable.

Another contribution of this work is the introduction of Rankzzy, a new multi-agent decision-making method for addressing the problems posed in our framework. It is designed to combine both quantitative and qualitative data in an explainable and transparent manner. The values involved in the problem are modelled as LR-fuzzy numbers so that we can study their impact and relevance using partial-order relationships. Moreover, we have removed the use of prior weights. Instead, we take the notion of decision weighting space, in which we can find out how valuable the actions of the problem are. Accordingly, we provide the stakeholders with a less restrictive decision solution. Moreover, we have shown the mathematical properties of our decision system in terms of consistency and the existence of solutions in any decision scenario. Hence, Rankzzy is intentionally designed to cover a great part of the existing gaps in the literature, adding value to the VBDM methodology.

\section{Related work}\label{s:relatedwork}

Our work is related to the literature on Multi-criteria Decision Making (MCDM) and Value-based Decision Making (VBDM). In particular, we tackle the VBDM problem by proposing a novel MCDM framework. Table~\ref{t:literature} summarises the literature in both MCDM and VBDM and its relationship to our work. The columns of the table refer to the requirements introduced in the previous section. That is, we highlight in each column the contributions found in the literature related to whether they followed a multi-agent systems (MAS) approach, they used qualitative and/or quantitative variables, they modelled uncertainty by using fuzzy logic, the work required partial orders and/or preference systems, and the elimination of fixed weights. As we show, our work is the first one to meet all the requirements in a unified manner. Next, we give an overview of the most related MCDM and VBDM works.

\begin{table*}[t!]
\centering
\caption{Comparison of the approaches used in literature review methods for addressing value-based decision problems. The ``-'' mark indicates when a criterion is not applicable.}
\label{t:literature}
{
    \begin{tabularx}{\linewidth}{clcccccc}
    \toprule
    
    &
    Work
    &
    MAS
    &
    Qualitative
    &
    Quantitative
    &
    Uncertainty
    &
    Order
    $\preceq$ 
    &
    Unweighted
    \\

    \midrule

    \multirow{16}{*}{\rotatebox[origin=c]{90}{Multiple-Criteria Decision-Making}} 
    & \cite{Chen2000}  & \xmark & \xmark & \checkmark & \checkmark  & \checkmark & \xmark \\ 
    & \cite{Fu2008}  & \xmark & \checkmark & \checkmark & \checkmark  & \checkmark & \checkmark \\
    & \cite{Guh2008}  & \xmark & \xmark & \checkmark & \checkmark  & \checkmark & \xmark \\
    & \cite{Opricovic2011}  & \xmark & \xmark & \checkmark & \checkmark  & \checkmark & \xmark \\
    & \cite{REZAEI2015}  & \xmark & \xmark & \checkmark & \xmark  & - & \checkmark \\
    & \cite{Afsordegan2016}  & \xmark & \checkmark & \xmark & \xmark  & \checkmark & \xmark \\
    & \cite{Pei2019}  & \xmark & \checkmark & \xmark & \checkmark  & \checkmark & \xmark \\
    & \cite{Liern2020}  & \xmark & \xmark & \checkmark & \xmark  & - & \checkmark \\
    & \cite{HOSSEINIDEHSHIRI2022} & \checkmark & \xmark & \checkmark & \xmark  & - & \checkmark \\
    & \cite{Bouslah2023} & \checkmark & \xmark & \checkmark & \xmark  & - & \checkmark \\
    & \cite{LopezGarcia2023} & \xmark & \xmark & \checkmark & \xmark  & - & \checkmark \\
    & \cite{Hajiaghaei2023} & \checkmark & \checkmark & \checkmark & \checkmark  & \checkmark & \xmark \\
    & \cite{TAVANA2024} & \xmark & \xmark & \checkmark & \checkmark  & \checkmark & \checkmark \\
    & \cite{Goldani2024} & \checkmark & \xmark & \checkmark & \checkmark  & \checkmark & \checkmark \\
    & \cite{Bas2024} & \xmark & \checkmark & \checkmark & \xmark  & - & \checkmark \\
    \midrule
    
    \multirow{15}{*}{\rotatebox[origin=c]{90}{Value-Based Decision-Making}} 
    & \cite{Lopez2017} & \xmark & \xmark & \checkmark & \xmark  & \xmark & \xmark \\
    & \cite{Serramia2018} & \xmark & \xmark & \checkmark & \xmark  & \checkmark & \xmark \\
    & \cite{murukannaiah2020} & \checkmark & \checkmark & \checkmark & \xmark & - & \xmark \\ 
    & \cite{Ajmeri2020} & \xmark & \xmark & \checkmark & \xmark & - & \xmark \\ 
    & \cite{Serramia2020} & \xmark & \checkmark & \xmark & \xmark  & \checkmark & - \\
    & \cite{Mashayekhi2022} & \checkmark & \checkmark & \xmark & \xmark & - & \xmark \\ 
    & \cite{Yurrita2022} & \checkmark & \checkmark & \checkmark & \xmark & - & \xmark \\ 
    & \cite{Lera2022} & \checkmark & \xmark & \checkmark & \xmark & \xmark & \checkmark \\
    & \cite{Serramia2023} & \xmark & \checkmark & \xmark & \xmark  & \checkmark & \checkmark \\
    & \cite{Serramia2024} & \checkmark & \checkmark & \xmark & \xmark  & \checkmark & \checkmark \\
    & \cite{LeraLeri2024} & \checkmark & \checkmark & \checkmark & \xmark & \xmark & \checkmark \\
    & \cite{Arias2024} & \xmark & \checkmark & \checkmark & \xmark  & - & \xmark \\
    & \cite{Osman2024} & \checkmark & \checkmark & \xmark & \xmark  & - & \xmark \\
    & \cite{LopezGarcia2024} & \xmark & \checkmark & \checkmark & \xmark  & - & \checkmark \\
    & \cite{shen2025} & \xmark & \checkmark & \checkmark & \xmark  & - & \xmark \\
    
    \midrule
    
    & {Rankzzy} & \checkmark & \checkmark & \checkmark & \checkmark  & \checkmark & \checkmark \\
    
    \bottomrule
    \end{tabularx}
}
\end{table*}

\subsection{Decision analysis}

Real-life decision-making scenarios are very complex and ill-structured. For that reason, researchers have attached this major concern to create a robust methodology for handling biased information. On the one hand, fuzzy logic was designed to deal with data uncertainty during the decision process \cite{Zadeh1965, Mardani2015}. The aggregation operators have proven to be very useful for ranking alternatives \cite{Fu2008, Guh2008, Afsordegan2016, Pei2019}, as well as the extension of classical MCDM techniques such as TOPSIS \cite{Chen2000} or VIKOR \cite{Opricovic2011}. Likewise, fuzziness is also helpful when assessing relevance over criteria by using linguistic variables to capture qualitative insights \cite{Abbasbandy2009, Hajiaghaei2023}. On the other hand, unweighted strategies were designed to relax the normative bias attached to fixed weighting schemes \cite{REZAEI2015, Liern2020}. In most cases, the set of weights is determined by an expert panel, thus leading to subjective results \cite{Jacquet1982, Ouenniche2018, Watrobski2019}. Unweighted decision-making uses decision variables as weights to compute a mathematical programming problem that produces both minimum and maximum scores for each alternative. This approach can also be implemented for MAS or multi-staged information \cite{HOSSEINIDEHSHIRI2022, Bouslah2023}. Since the resulting optimisation problem may give rise to objections among stakeholders, some authors have generated solutions to extract decisional weights as a unified evaluation vector \cite{LopezGarcia2023}.

It is worth noting that fuzzy and unweighted techniques were designed to cover some of the main drawbacks of decision-making theory. Although some authors have addressed this problem \cite{Fu2008, TAVANA2024, Goldani2024}, the contributions of both theories are independent of each other. For that reason, one of the motivations of this work is to combine both techniques for approaching the value-based decision-making problem. In the same way, we can see that few works in the literature consider the incorporation of quantitative and qualitative variables simultaneously. It is therefore crucial to explore this gap to add value to the applicability of this field.

\subsection{Value-based decision-making}

The use of ranking systems is a common approach for solving value-based decision-making problems \cite{Roth1992}. It is assumed that the existence of preferences, either partially known by stakeholders or obtained by multi-objective optimisation, for the different actions involved in a given problem guides the outcome. From a value system, the main objective is to determine either hierarchical or partial order relationships that will establish the preferences among moral actions \cite{Lopez2017, Serramia2018, Serramia2020}. Within this value system, the ranking is usually induced by a combination of utility functions that consider the relation between norms and values \cite{Serramia2023, Serramia2024}, where sophisticated regression techniques have also been applied to aggregate multi-value systems \cite{Lera2022, LeraLeri2024}. Other authors also highlight the idea of structuring a value-aligned decision problem using decision matrices \cite{Ajmeri2020, LopezGarcia2024}. Although the order theory is an effective approach with straightforward interpretation, the main assumption that moral values can be qualitatively measured is rather controversial since evidence supports that people tend to rely on fuzzy intuition for moral reasoning \cite{Reyna2009}. Furthermore, the choice of norm relationships and ethical considerations may be subject to inconsistencies \cite{Serramia2018}, as well as the one-to-one value assessment \cite{LopezGarcia2024}. Therefore, there is a need for flexible methodologies that can model the underlying behaviour of human values and their actual impact on the final result.

In the same line as MCDM, the approaches have been conducted considering values as quantitative \cite{Lopez2017, Serramia2018, Ajmeri2020, Lera2022} or qualitative \cite{Serramia2020, Serramia2023} attributes. We can also find works in which the algorithms consider both kinds of variables \cite{murukannaiah2020, Mashayekhi2022, LopezGarcia2024}. Even though these techniques can be used to raise the problem, a framework that effectively combines both types of data in the same decision scenario is required.

Regarding the interpretation of values, the proposed approaches do not consider the subjective/vagueness attached to them. The main assumption that moral values can be formally measured is rather controversial since evidence supports that people tend to rely on fuzzy intuition for moral reasoning \cite{Reyna2009}. Furthermore, the choice of norm relationships and ethical considerations may be subject to inconsistencies \cite{Serramia2018}. Therefore, there is a need for flexible methodologies that can model the different notions of values and the comparisons among the distinct values in each problem. The first consideration can be addressed using fuzzy logic. For the second one, we just need an order relationship defined over the same fuzzy structure.

\section{Preliminaries}\label{s:preliminaries}

\subsection{Multiple-criteria decision-making}

Multiple-criteria decision-making (MCDM) is a major discipline of operations research composed of a set of algorithms and methodologies whose goal is to aid the decision-maker when there exist multiple conflicts of interest \cite{Hwang1981}. The major advantage of its implementation in applied decision-making problems is that it combines an experience-based background based on mathematical reasoning.

Real-life situations are too complex to be analysed by a single criterion \cite{Doumpos2013}. MCDM supports decision-makers in considering all the possible underlying factors involved in the problem.  The MCDM paradigm usually considers a decision matrix $X = [x_{ij}]_{ij}$ and a weighting scheme $w=(w_1,\dots,w_M)$ so that $i\in\{1,\dots, N\}$ and $j\in\{1,\dots, M\}$ indicates the number of alternatives and criteria respectively. Each component of the decision matrix indicates an attribute or performance value, while each weight indicates the relevance of each criterion during the decision process. By definition, the sum of weights is equal to one, thus showing the total relevance of the criteria within the decision scenario. Some MCDM techniques are defined over $L^p$ metric spaces to ascertain the demands of decision-makers \cite{Triantaphyllou2000}. The comparative strategies can be tuned with the aim of offering pessimistic or optimistic approaches that the multi-agents have to carry out.

Depending on the purpose of the decision-makers, the multi-agent system has to be designed with a particular multicriteria technique that meets their needs \cite{Doumpos2013}. For that reason, MCDM approaches exist to solve conflict-of-interest problems using different strategies.

The main advantage that offers this approach in VBDM problems is that we can set a preference system using a structure of actions $(A_1, \dots, A_N)$ evaluated by values $(V_1, \dots, V_M)$ through the assessment of $(Ag_{1}, \dots, Ag_{K})$ agents. The way we construct the decision matrix for this work is described in Section~\ref{s:decision_matrix}.

\subsection{Fuzzy logic}

When we refer to certainty, we assume the structures and parameters of a mathematical model to be definitely known and that there are no doubts about their values or occurrence \cite{Zimmermann2001}. However, when dealing with subjective criteria, such as ethics or values, we cannot ensure the reliability of numerical expressions. In order to address this major concern, \cite{Zadeh1965} developed fuzzy logic to model the uncertainty that involves real-world events mathematically. Its application in decision theory is based on the use of fuzzy sets instead of numerical values defined by membership functions. 

The conventional MCDM approaches tend to be less effective in dealing with the intrinsic uncertainty of decision theory \cite{Kahraman2003, Kahraman2006}. The assessment of subjective and qualitative events in nature is very difficult for the decision-maker since it is hard to express preferences using exact numerical values \cite{Doumpos2013}. In most cases, agents must deal with the imprecise or vague nature of the linguistic assessment given by decision-makers. For that reason, many intelligent systems use fuzzy sets to approximate the non-objective features that involve the decision. In this section, we define some basic fuzzy notions commonly used in fuzzy approaches.

\begin{definition}[Fuzzy set \cite{Zadeh1965}]\label{def:fuzzy_set}
Let $\mathscr{X}$ be a space of points or objects. A fuzzy set $A$ in $\mathscr{X}$ is characterized by a membership function $\mu_A$ which associates with each point $x$ in $\mathscr{X}$ a real number in the interval $[0,1]$, with the value of $\mu_A$ at $x$ representing the ``grade of membership'' of $x$ in $A$.
\end{definition}



\begin{definition}[LR-fuzzy number \cite{Dubois1978}]\label{def:lr_fuzzy}
An LR-fuzzy number $\tilde{X}$ is a 4 component array $\tilde{X} = (x^L, x^R, \alpha^L, \alpha^R)_{L,R}$ whose membership function has the following form:
\begin{equation}\label{eq:fuzzy_membership}
    \mu_{\tilde{X}}(t) = 
    \left \{
    \begin{array}{lcl}
    \displaystyle L\left(\frac{x^L - t}{\alpha^L}\right) & \mbox{if} & t < x^L, \\
    1                                                    & \mbox{if} & x^L \le r \le x^R, \\
    \displaystyle R\left(\frac{t - x^R}{\alpha^R}\right) & \mbox{if} & t > x^R, \\
    \end{array}
    \right.
\end{equation}
where $L, R: \R\to[0,1]$ are the so-called reference functions which are decreasing in the support of $\tilde{X}$ and upper semi-continuous with $L(0)=R(0)=1$. For every $\alpha\in[0,1]$, we define its $\alpha$-cut as $A_{\alpha}=\left\{x\in A : \mu_{A}(x)\ge\alpha \right\}$.
\end{definition}


\begin{definition}[Fuzzy arithmetic \cite{Dubois1978}]\label{def:fuzzy_operations}
    Given three LR-fuzzy numbers $\tilde{X}$, $\tilde{X}_1$ and $\tilde{X}_2$ with bounded support and positive components and a real number $\lambda$, we can define the following fuzzy operations:
    \begin{equation}\notag
    \begin{array}{rl}
        \tilde{X}_1 \oplus \tilde{X}_2  = &
        \displaystyle \left(x_1^L+x_2^L, x_1^R+x_2^R, \alpha_1^L+\alpha_2^L, \alpha_1^R+\alpha_2^R\right)_{L_1+L_2,R_1+R_2},\vspace{1mm} \\
        
        \tilde{X}_1 \otimes \tilde{X}_2  = &
        \displaystyle \left(x_1^L x_2^L,\ x_1^R x_2^R,\ \alpha_1^L \alpha_2^L,\ \alpha_1^R\alpha_2^R\right)_{L_1L_2,R_1R_2},\vspace{1mm} \\
    
        \lambda \odot \tilde{X}  = &
        \displaystyle \left(\lambda x^L, \lambda x^R, \lambda\alpha^L, \lambda\alpha^R\right)_{L,R},\vspace{1mm} \\
        
        \ominus\tilde{X}  = &
        \displaystyle \left(-x^R,\ -x^L,\ -\alpha^R,\ -\alpha^L\right)_{R,L},\vspace{1mm} \\
        
        1/\tilde{X} = & 
        \displaystyle \left(1/x^R,\ 1/x^L,\ 1/\alpha^R,\ 1/\alpha^L\right)_{R,L},\vspace{1mm} \\
        
        \tilde{X}_2 \oslash \tilde{X}_1  = &
        \displaystyle \tilde{X}_2 \otimes 1/\tilde{X}_1,
        \vspace{1mm} \\
        
        \tilde{X}^{\lambda} = & 
        \displaystyle \left((x^L)^{\lambda}, (x^R)^{\lambda}, (\alpha^L)^{\lambda}, (\alpha^R)^{\lambda}\right)_{L,R}. \\
    \end{array}
    \end{equation}
\end{definition}


We can also define some basic operators for their respective membership functions $\mu_{\tilde{X}_1}$ and $\mu_{\tilde{X}_2}$:
\[
    \begin{array}{lcll}
        \hbox{Intersection:} & \mu_{\tilde{X}_1\cap\tilde{X}_2}(t) & = &
        \min\left\{\mu_{\tilde{X}_1}(t_1), \mu_{\tilde{X}_1}(t_2)\right\}, \\
        
        \hbox{Union:}        & \mu_{\tilde{X}_1\cup\tilde{X}_2}(t) & = &
        \max\left\{\mu_{\tilde{X}_1}(t_1), \mu_{\tilde{X}_1}(t_2)\right\}, \\
        
        \hbox{Complement:}   & \mu_{\tilde{X}_1^{\mathbf{C}}}(t) & = & 1 - \mu_{\tilde{X}_1}(t). \\
    \end{array}
\]

\subsection{Fuzzy order relationships}\label{s:ranking}

Fuzzy order relations have been widely studied since Zadeh introduced the concept back in 1971. A major advantage of using fuzzy sets for modelling values is that we can capture their subjective nature. In most cases, values are described through linguistic variables or subjective preference scales rather than precise numerical parameters. However, it is worth mentioning that fuzzy numbers are not ordered per se. Then, in this section, we introduce the definitions and results that will lead us to the ranking of values and actions.

\begin{definition}[$\lesssim$ partial order relationship \cite{Dubois1978}]\label{def:fuzzy_order}
Let $\tilde{X}_1$ and $\tilde{X}_2$ be two LR-fuzzy numbers. Then,
\begin{equation}\label{eq:fuzzy_order}
    \tilde{X}_1\lesssim\tilde{X}_2 \Longleftrightarrow \tilde{X}_1\vee\tilde{X}_2 = \tilde{X}_2,
\end{equation}
so that the LR-fuzzy number $\tilde{X}_1\vee\tilde{X}_2$ has the following membership function:
\begin{equation}\label{eq:fuzzy_order2}
    \mu_{\tilde{X}_1\vee\tilde{X}_2}(t) = \sup_{t=t_1\vee t_2}\left\{ \mu_{\tilde{X}_1}(t_1)\wedge\mu_{\tilde{X}_1}(t_2) \right\}.
\end{equation}
\end{definition}

From (\ref{eq:fuzzy_order}) we can formulate fuzzy optimisation problems by using the $\lesssim$ order relationship \cite{Ramik1985}. The properties of this partial order relationship are as follows:

\begin{proposition}
   Given three arbitrary LR-fuzzy numbers $\tilde{X}_1,\tilde{X}_2,\tilde{X}_3\in\mathscr{X}$, then:
   \begin{enumerate}
       \item Reflexivity: $\tilde{X}_1 \lesssim \tilde{X}_1$.
       
       \item Transitivity: If $\tilde{X}_1 \lesssim \tilde{X}_2$ and $\tilde{X}_2 \lesssim \tilde{X}_3$, then $\tilde{X}_1 \lesssim \tilde{X}_3$.
       
       \item Antisymmetry: If $\tilde{X}_1 \lesssim \tilde{X}_2$ and $\tilde{X}_2 \lesssim \tilde{X}_1$, then $\tilde{X}_1 = \tilde{X}_2$.
   \end{enumerate}
\end{proposition}

\begin{lemma}\label{lemma:fuzzy_order_alpha_cut}
Given two arbitrary LR-fuzzy numbers $\tilde{X}_1,\tilde{X}_2\in\mathscr{X}$. Then, $\tilde{X}_1\lesssim\tilde{X}_2$ if, and only if, for each $\alpha\in[0,1]$:
\begin{equation}\label{eq:fuzzy_order_alpha_cut}
    \begin{array}{c}
        \inf\{\tilde{X}_{1,\alpha}\} \le \inf\{\tilde{X}_{2,\alpha}\}, \vspace{1mm}\\
        \sup\{\tilde{X}_{1,\alpha}\} \le \sup\{\tilde{X}_{2,\alpha}\}.
    \end{array}
\end{equation}

\end{lemma}

\begin{proof}
$(\Rightarrow)$ We define $\tilde{Z} = \tilde{X}_1\vee\tilde{X}_2$. Then, according to (\ref{eq:fuzzy_order}) and (\ref{eq:fuzzy_order2}), for each $\alpha\in[0,1]$ it is hold:
\begin{equation}
    \begin{array}{c}
       \displaystyle\inf\{Z_{\alpha}\} =
       \inf\left\{x\in\mathscr{X}: \sup_{x=t_1\vee t_2}\left\{ \mu_{\tilde{X}_1}(t_1)\wedge\mu_{\tilde{X}_1}(t_2) \right\} \right\}, \vspace{1.5mm}\\
       
       \displaystyle\sup\{Z_{\alpha}\} =
       \sup\left\{x\in\mathscr{X}: \sup_{x=t_1\vee t_2}\left\{ \mu_{\tilde{X}_1}(t_1)\wedge\mu_{\tilde{X}_1}(t_2) \right\} \right\}.
    \end{array}
\end{equation}

Therefore, we have:
\begin{equation}
    \begin{array}{c}
       \inf\{Z_{\alpha}\} = \inf\{\tilde{X}_{2,\alpha}\} \ge \inf\{\tilde{X}_{1,\alpha}\},\vspace{1mm}\\
       \sup\{Z_{\alpha}\} = \sup\{\tilde{X}_{2,\alpha}\} \ge \sup\{\tilde{X}_{1,\alpha}\}.
    \end{array}
\end{equation}

$(\Leftarrow)$ If both equation of (\ref{eq:fuzzy_order_alpha_cut}) are verified, then we can say that $\tilde{X}_2$ dominates $\tilde{X}_1$ from both the left and the right side. Hence, we note that $\tilde{X}_1\lesssim\tilde{X}_2$.
\end{proof}

Even though $\lesssim$ is well-defined in mathematical terms, this order relationship can lead to indetermination. For that reason, it is more convenient to use the following definition.

\begin{definition}[Vertex method \cite{Chen2000}]
Let $\tilde{X}_1$ and $\tilde{X}_2$ be two LR-fuzzy numbers. Then, the vertex method between $\tilde{X}_1$ and $\tilde{X}_2$ is the following distance:
\begin{equation}\label{eq:vertex_method}
    d(\tilde{X}_1, \tilde{X}_2) = \big|\big| \tilde{X}_1 \ominus \tilde{X}_2 \big|\big|_2,
\end{equation}
so that the Euclidean distance is calculated over the vertices of the $0$ and $1$ $\alpha$-cuts of the fuzzy numbers, i.e., the kernel and core of their respective membership functions.
\end{definition}

\begin{definition}[Vertex order relationship $\preceq$]
The fuzzy order relationship $\preceq$ is a preorder relationship induced by the vertex method over the fuzzy $\fuzzyzero$.
\end{definition}

Hence, given two LR-fuzzy numbers $\tilde{X}_1$ and $\tilde{X}_2$, we say:
\begin{equation}
    \tilde{X}_1 \preceq \tilde{X}_2 \Longleftrightarrow \left|\left|\tilde{X}_1\right|\right|_2 \le \left|\left|\tilde{X}_2\right|\right|_2,
\end{equation}
thus taking into account the four vertices that determine the LR-fuzzy structure.

\begin{lemma}\label{lemma:fuzzy_order_distance}
Given two arbitrary LR-fuzzy numbers $\tilde{X}_1,\tilde{X}_2\in\mathscr{X}$ so that $x_1^L - \alpha_1^L$, $x_2^L - \alpha_2^L > 0$. Then:
\begin{equation}\label{eq:fuzzy_order_distance}
    \tilde{X}_1\lesssim\tilde{X}_2 \Longrightarrow \tilde{X}_1\preceq\tilde{X}_2.
\end{equation}
\end{lemma}
\begin{proof}
    According to Lemma~\ref{lemma:fuzzy_order_alpha_cut}, if $\tilde{X}_1\preceq\tilde{X}_2$ we know that:
    \begin{equation}
    \left\{
    \begin{array}{c}
        \inf\{\tilde{X}_{1,\alpha}\} \le \inf\{\tilde{X}_{2,\alpha}\}, \vspace{1mm}\\
        \sup\{\tilde{X}_{1,\alpha}\} \le \sup\{\tilde{X}_{2,\alpha}\}.
    \end{array}
    \right.
\end{equation}
Then, for the cases when $\alpha\in\{0,1\}$ it is hold that:
    \begin{equation}
    \left\{
    \begin{array}{ccc}
        x_1^L \le x_2^L & \hbox{and} & x_1^L - \alpha_1^L \le x_2^L - \alpha_2^L, \vspace{1mm}\\
        x_1^R \le x_2^R & \hbox{and} & x_1^L + \alpha_1^R \le x_2^R + \alpha_2^R.
    \end{array}
    \right.
\end{equation}
Therefore, when we compute the $L^2$ distance, we obtain:
\begin{equation}
\begin{array}{rcl}
\left|\left|\tilde{X}_1\right|\right|_2 = &
\left|\left|\left(x_1^L - \alpha_1^L, x_1^L, x_1^R, x_1^R + \alpha_1^R \right)\right|\right|_2 & \le \vspace{1mm} \\
                \le & \left|\left|\left(x_2^L - \alpha_2^L, x_2^L, x_2^R, x_2^R + \alpha_2^R \right)\right|\right|_2 & = \left|\left|\tilde{X}_2\right|\right|_2.
\end{array}
\end{equation}
As a result, we conclude that $\tilde{X}_1\preceq\tilde{X}_2$.
\end{proof}

The partial order relationships defined over fuzzy structures play a key role in the approach presented in this paper for VBDM problems. Since values and action evaluations are represented as a series of fuzzy numbers, it is essential to define a procedure that orders them. For this reason, in our work, it is advisable to adopt a stricter order relation than $\lesssim$ because of its comparative limitations. Here, the partial-order relationship $\preceq$ introduced can be used for ranking. It indicates that the farther the distance from the fuzzy origin of the coordinate system, the more preferred the action.

\section{The Fuzzy Unweighted Value-Based Decision-Making Framework}\label{s:problem_statement}

The main goal of value-based decision-making is to determine preferences among different actions to perform in a strategic choice environment by considering a human value system that impacts the outcome. This problem requires a MAS approach in which several relevant factors must be considered properly. To this end, in this work, we introduce a theoretical framework for computing the trade-offs between competing values by considering values of different natures (quantitative and qualitative). The components of our proposal deal with the uncertainty of the criteria, together with an optimisation stage that gives us insightful information about the balance and variation of actions when evaluating them. The fuzzy logic helps us to model the nature of values from data samples, and the unweighted approach is intended to eliminate the assignment of biased weighting schemes. Finally, the aggregation technique allows us to generate highly customisable options to meet the demands of the stakeholders.

\subsection{Problem Statement}

The problem statement we introduce in this paper consists of several elements that ensure an end-to-end definition of the decision system. To focus the casuistry on stakeholders, we have considered the tuple with the parties that allow us to both establish and solve the problem.

\begin{definition}[FUW-VBDM problem]
    A fuzzy unweighted value-based decision-making problem is defined as a tuple: 
    \begin{equation}
        \left(
              Ag,
              A,
              V,
              \preceq,
              \tilde{X},
              \tilde{\Omega}_{lu},
              \tilde{\M}^{p}
        \right)
    \end{equation}
    so that it is composed of agents ($Ag$), actions ($A$), values ($V$), fuzzy order relationship ($\preceq$), decision matrix ($\tilde{X}$), the domain of weights ($\tilde{\Omega}_{lu}$) and a score function ($\tilde{\M}^{p}$) based on a $p$-mean with $p\in\R$. The solution to this problem is a ranking-based preference system regarding the actions.
\end{definition}

Two essential contributions to handling the FUW-VBDM problem are the use of a decision matrix composed of both quantitative and qualitative attributes and a constrained optimisation stage of the score function over the space of weighting schemes.

\begin{example}
Let us assume a university wants to standardise exam methods across faculties and starts a consultation about the best exam method. For the sake of simplicity, let us assume the exam alternatives are: multiple-choice (MC) and open-answer (OA) exams. The university wants to consider two values when making that decision: fairness (F) and cost (C). Obviously, fairness is not only difficult to quantify but also ambiguous, whereas cost can be quantified, but it is subjective.
\end{example}

\subsection{Determination of the decision matrix}\label{s:decision_matrix}

The paradigm of a value-based decision-making problem can be understood as a scenario in which different actions, assessed with moral values, are evaluated to ascertain the most appropriate or preferred one. To this end, the statement of the problem requires that different agents $(Ag)$ determine a set of actions $(A)$ evaluated on values $(V)$. As we have to make inferences about the values, we need to determine various assessments for each moral value. We indicate the ethical valuation of a value within an action by means of a valuation scheme.

\begin{definition}[Agent assessment vector]
    Given two sets of actions and values, we define the agent assessment vector for the $i\in A$ action and the $j\in V$ value as a $|Ag|$-dimensional vector:
    \begin{equation}\label{eq:assessments_vector}
        \Gamma_{ij} = \left(\gamma^{1}_{ij}, \dots, \gamma^{|Ag|}_{ij}\right)\in[0,1]^{|Ag|}.
    \end{equation}
    Thus, $\Gamma_{ij}$ assesses each value in $|Ag|$ disjoint competencies.
\end{definition}

The assessment vectors can be interpreted as a way of determining the impact of a particular value within an action. In fact, each $\Gamma_{ij}$ is composed of the valuations of the multi-agent system. Such assessment can be understood as the rating of an expert panel or the percentage breakdown of association with certain qualities.

In our proposal, we can consider values that are both qualitative and quantitative. For that reason, we define different information fusions for aggregating the agent assessment vectors depending on their nature.


\begin{definition}[Quantitative transformation]\label{def:qt_transform}
Let $\Gamma_{ij} = \left(\gamma^{1}_{ij}, \dots, \gamma^{|Ag|}_{ij}\right)$ the ${ij}$-assessment vector of a quantitative value. The transformation of $\Gamma_{ij}$ into the ${(i,j)}\in A\times V$ entry of the decision matrix is defined as the following trapezoid fuzzy number:
    \begin{equation}\label{eq:qt_transform}
        \Psi_{\text{QT}}(\Gamma_{ij}) = 
        \left(
        \min\Gamma_{ij},\
        \bar{\Gamma}_{ij} - \sigma_{\Gamma_{ij}},\
        \bar{\Gamma}_{ij} + \sigma_{\Gamma_{ij}},\
        \max\Gamma_{ij}
        \right)
    \end{equation}
where $\bar{\Gamma}_{ij}$ and $\sigma_{\Gamma_{ij}}$ stand for the mean and deviation according to the $k\in Ag$ agent.
\end{definition}

\begin{definition}[Qualitative transformation]\label{def:ql_transform}
Let $\Gamma_{ij} = \left(\gamma^{1}_{ij}, \dots, \gamma^{|Ag|}_{ij}\right)$ the ${ij}$-assessment vector of a qualitative value. The transformation of $\Gamma_{ij}$ into the ${ij}$-entry of the decision matrix is defined as the following trapezoid fuzzy number:
    \begin{equation}\label{eq:ql_transform}
    \Psi_{\text{QL}}(\Gamma_{ij}, \tilde{\Lambda}_{ij}) = 
    \bigoplus_{c=1}^{C}\left(\frac{1}{|Ag|}\sum_{k\in Ag}\mathbf{1}\left[\gamma_{ij}^{k} = c\right]\right)\odot\tilde{\Lambda}_{ij}^{c},
    \end{equation}
    so that $c\in\{1,\dots,C\}$ represents the categories or labels of the qualitative scale and $\Lambda_{ij} = \left(\tilde{\Lambda}_{ij}^{1}, \dots, \tilde{\Lambda}_{ij}^{C}\right)$ the fuzzy correspondences of each category. The $\mathbf{1}[\cdot]$ operator stands for the Iverson bracket.
\end{definition}

From these transformations, we can generate the fuzzy trapezoids that will allow us to know the valuation of each of the values for every action.

\begin{example}
Once the problem is posed, the fairness of each exam method is determined by performing a survey where the university academics state if the type of test is fair, neutral or unfair. The teachers' answers are broken down as follows:
\begin{center}
    \begin{tabular}{rlcl}
        {                       } & Fair     & Neutral & Unfair \\
        $\Gamma_{\text{MC, F}} =$ & (0.07,   & 0.45,   & 0.48), \\
        $\Gamma_{\text{OA, F}} =$ & (0.79,   & 0.18,   & 0.03).
    \end{tabular}
\end{center}


The aggregated fairness, which in our case is represented as a fuzzy qualitative value, is modelled using the following fuzzy correspondences for the different survey answers:

\begin{figure}[ht!]
    \centering
    \begin{tikzpicture}
        \centering
        \begin{axis}[
            width=0.6\linewidth,
            axis lines = left,
            ytick={0, 0.5, 1},
            xmin=0,
            xmax=1.05,
            ymin=0,
            ymax=1.1,
            yscale=.5,
            axis equal image = true,
            ]
        \addplot[red, mark=*, line width=1.25pt, mark options={mark size=2.5pt}]
            coordinates {(0.0, 1) (0.2, 1) (  0.4, 0)};
        \node[red] at (axis cs: 0.1, 0.8) {Unfair};
        
        \addplot[orange, mark=*, line width=1.25pt, mark options={mark size=2.5pt}]
            coordinates {(0.2, 0) (0.4, 1) (0.6, 1) ( 0.8, 0)};
        \node[orange] at (axis cs: 0.5, 0.8) {Neutral};
    
        \addplot[green, mark=*, line width=1.25pt, mark options={mark size=2.5pt}]
            coordinates {(0.6, 0) (0.8, 1) (1, 1)};
        \node[green] at (axis cs: 0.9, 0.8) {Fair};
    
        \addplot[black, dashed, mark=none, opacity=.5] coordinates {(1, 1) (1, 0)};
        
        \end{axis}
    \end{tikzpicture}
    \caption{Fuzzy correspondences of the three possible survey answers.}
\end{figure}
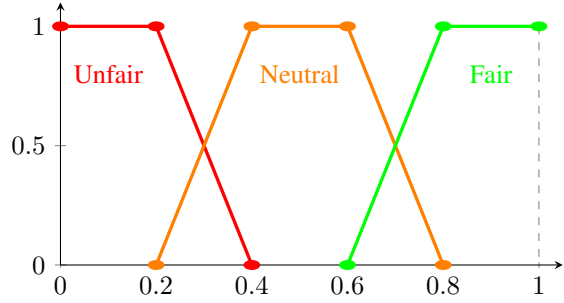

The time costs -- which have a quantitative nature -- are measured by the time involved in preparing, setting, and marking the test. After timing teachers on both tasks, where the sample size is $|Ag|=100$, we obtain the density distribution displayed in Figure~\ref{fig:density}. From such sampling, we transform both densities into LR-fuzzy numbers as in (\ref{def:qt_transform}). Thus, we make full use of the entire range of values obtained.

\begin{figure}[ht!]
  \centering
  \includegraphics[width=0.6\linewidth]{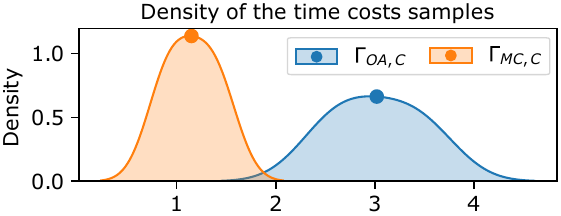}
  \caption{Density plot of the cost features for the open answer and multiple-choice tests measured in hours. The ``$\bullet$'' ticks indicate the mean of each distribution.}
  \label{fig:density}
\end{figure}

\end{example}

The last step to perform is the aggregation of both assessment vectors and their fuzzy correspondences. The following definition provides the aggregation strategy.

\begin{definition}[Fuzzy decision matrix]
    Let $\Gamma_{ij} = \left(\gamma^{1}_{ij}, \dots, \gamma^{|Ag|}_{ij}\right)$ the $ij$-th assessment vector and $\tilde{\Lambda}_{ij} = \left(\tilde{\Lambda}_{ij}^{1}, \dots, \tilde{\Lambda}_{ij}^{C}\right)$ the respective fuzzy correspondences when qualitative. We define the fuzzy decision matrix that evaluates the $i$ action on the $j$ value as:
    \begin{equation}\label{eq:aggregated_fuzzy_trapezoid}
        \tilde{X}_{ij} = 
        \begin{cases}
            \Psi_{\text{QT}}(\Gamma_{ij}),                       & \text{if } j\in V \text{ quantitative},\\
            \Psi_{\text{QL}}(\Gamma_{ij}, \tilde{\Lambda}_{ij}), & \text{if } j\in V \text{ qualitative}.
        \end{cases}
    \end{equation}
\end{definition}



\begin{example}
According to (\ref{eq:aggregated_fuzzy_trapezoid}) and the assessment vectors previously shown, we can compute, for instance, the entry associated with the fairness of the multiple choice test. In other words, we compute the $\tilde{X}_{\text{MC, F}}$ entry of the decision matrix. As this value is qualitative, we proceed as follows:
\[
\begin{array}{rl}
    \tilde{X}_{\text{MC, F}} & = 0.07\odot(0.0, 0.0, 0.2, 0.4) 
     \oplus 0.45\odot(0.2, 0.4, 0.6, 0.8) 
     \oplus 0.48\odot(0.6, 0.8, 1.0, 1.0) = \\
     & = (0.132,  0.236,  0.436,  0.622).
\end{array}
\]

If we proceed in the same way for each of the remaining combinations, we obtain the fuzzy decision matrix of Table~\ref{t:matrix_random_example}.
\begin{table}[ht!]
    \centering
    \caption{Aggregated fuzzy decision matrix for the test selection problem.}
    \label{t:matrix_random_example}
    \begin{tabularx}{\linewidth}{lrr}
        \toprule
        Test &
        \multicolumn{1}{c}{Fairness} &
        \multicolumn{1}{c}{Costs} \\
        \midrule
        MC &  $(0.132,  0.236,  0.436,  0.622)$ &  $(2.501,  2.808,  3.208,  3.564)$ \\
        OA &  $(0.510,  0.704,  0.904,  0.946)$ &  $(0.846,  1.032,  1.266,  1.465)$ \\
        \bottomrule
    \end{tabularx}
\end{table}
\end{example}

\subsection{Domain of Bounded Fuzzy Weights}\label{s:domain_weights}

Every data-driven technique for decision-making requires a weighting scheme to determine the relative impact of the attributes in the utility function. For that reason, it is essential to consider it when applying a specific algorithm. The role of weights is to determine the relative importance of the criteria selected for a particular problem. However, we cannot neglect the sensitivity attached to them. In this paper, we present an unweighted method for solving such a concern. The core idea is to define a weighting space by means of a set of lower and upper bounds $\{(l_j,u_j)\}_{j\in V}$ that sets the constraints subjected to values \cite{mitesis2023}.

\begin{definition}[Fuzzy bounds]\label{def:fuzzy_bounds}
    Given a finite set of fuzzy numbers $\tilde{A} = \{\tilde{a}_j\}_{j\in V}\subseteq\mathscr{X}$, we define $\tilde{l} = \{\tilde{l}_j\}_{j\in V}$ and $\tilde{u} = \{\tilde{u}_j\}_{j\in V}$ as the lower and upper bounds of $\tilde{A}$ if and only if $\tilde{l}_j \lesssim \tilde{a}_j \lesssim \tilde{u}_j$, $\forall j\in V$.
\end{definition}

According to the Definition~\ref{def:fuzzy_bounds}, the weights of the decision-making problem can be considered as variables in a mathematical constrained optimisation problem subject to the restriction that they belong to the feasible set of weights. To be more specific, if we consider two bounds $\tilde{l}_{j}$ and $\tilde{u}_{j}$ we can generate a space of weights $\tilde{\Omega}_{lu}$ as a bounded subset of $\mathscr{X}$ so that we can ensure the Zadeh's Extension Principle \cite{deBarros2017}, i.e.:
\begin{equation}\label{def:omega_bounds_constraints}
    \sum_{j\in V} {w}_j^L \le 1 \le \sum_{j\in V} {w}_j^R \
    \hbox{(constraints)}
    \quad
    \hbox{ and }
    \quad
    \fuzzyzero\lesssim\tilde{l}_j\lesssim\tilde{u}_j\lesssim\fuzzyone \
    \hbox{(bounds)},
\end{equation}
for each value $j\in V$.

\begin{definition}[Domain of bounded fuzzy weights]
Let $\tilde{l}$ and $\tilde{u}$ two sets of LR-fuzzy bounds provided $\fuzzyzero\lesssim\tilde{l}_j\lesssim\tilde{u}_j\lesssim\fuzzyone$ for all $j\in V$, we define the domain of possible weighting schemes in $\mathscr{X}$ as:
\begin{equation}\label{def:omega_fuw}
    \tilde{\Omega}_{lu}\triangleq
    \left\{
    \left[\wfuzzy\right]_{j\in V}\subseteq\mathscr{X}
    \Bigg |
    \sum_{j\in V} {w}_j^L \le 1 \le \sum_{j\in V} {w}_j^R, \
    \tilde{l}_j\lesssim\tilde{w}_j\lesssim\tilde{u}_j,\
    \forall j\in V
    \right\}.
\end{equation}
\end{definition}

\begin{example}
Following the stakeholders' guidelines, the fairness value has to have a major role in the final decision. Hence, we have to build a domain of bounded fuzzy weights to satisfy such a constraint. The domain is calculated using the bounds presented in Table~\ref{t:bounds_random_example}.
\begin{table*}[ht!]
    \centering
    \caption{Fuzzy lower and upper bounds for the test selection problem.}
    \label{t:bounds_random_example}
    \begin{tabularx}{\linewidth}{XXX}
        \toprule
        {Bound} &
        {Fairness} &
        {Cost} \\
        \midrule
        $\tilde{l}$ & $(0.60, 0.70 , 0.80 , 0.90)$ &  $(0.10, 0.15, 0.20, 0.25)$ \\
        $\tilde{u}$ & $(0.75, 0.80 , 0.90 , 1.00)$ &  $(0.15, 0.20, 0.25, 0.30)$ \\
        \bottomrule
    \end{tabularx}
\end{table*}
\end{example}

\begin{theorem}[Generalisation of the weighting space]
    For any VBDM problem, there exist two sets of fuzzy bounds $\tilde{l}$ and $\tilde{u}$ such that the problem can be stated through a domain of bounded fuzzy weights $\tilde{\Omega}_{lu}$.
\end{theorem}

\begin{proof}
    Let us consider a VBDM problem with $A$ actions and $V$ values. For the crisp case, the decisional weights can be defined as an $|V|$-dimensional vector $w\in\R^{|V|}$ so that $\sum_{j\in V} w_j=1$. Then, if we select $l=\min\{w_j : j\in V\}$ and $u=\max\{w_j : j\in V\}$ the VBDM problem attached is defined in $\tilde{\Omega}_{lu}$. For the non-crisp case, the fuzzy decisional weights $\tilde{w}$ consist of $|V|$ fuzzy numbers. For Zadeh's Extension Principle, we know that the sum of the left spreads is lower or equal to one and the sum of the right spreads is lower or equal to one, w.r.t each weight component. Then, if we select $\tilde{l} = \bigwedge_{j\in V} \tilde{w}_j$ and $\tilde{u} = \bigvee_{j\in V} \tilde{w}_j$ the VBDM problem attached is defined in $\tilde{\Omega}_{lu}$.
\end{proof}

From the above theorem, we can deduce that any VDBM problem can be addressed by using the concept of a space of feasible solutions contained in the domain of bounded fuzzy weights. It produces an implementation-oriented representability result for stakeholders when human values are part of the decision-making process. Thus, the need to establish spaces for decisional weights is highlighted, since it is latent that policy-makers use this method, even if implicitly.

\subsection{Generalised Fuzzy Mean as Score Function}

From a given dataset, the decision-making analysis described requires an aggregation strategy that produces a final score. This score will generate the ranking that will determine the preferences among the actions. Its definition must be customisable to meet the needs of the stakeholders, regardless of the scenario they find themselves in. On this account, we propose a generalised fuzzy $p$-mean to ensure their applicability, understanding, explainability, and transparency in decision-making.

\begin{definition}[Fuzzy $p$-mean score]
    Given a fuzzy decision matrix $[\tilde{x}_{ij}]\in\mathscr{X}^{|A|\times |V|}$, a weighting scheme $\tilde{w}\in\tilde{\Omega}_{lu}$ and a power $p\in\R\backslash\{0\}$, we define the fuzzy $p$-mean score as the function for each action $i\in A$ as follows:
    \begin{equation}\label{eq:FWMM}
    \begin{array}{cccc}
        \tilde{\M}^p_i: & \mathscr{X}^{|A|\times |V|}\times\tilde{\Omega}_{lu} & \to & \mathscr{X} \\
         & \left( [\tilde{x}_{ij}], \tilde{w} \right) &
         \mapsto& \displaystyle\left[\bigoplus_{j\in V} \tilde{w}_j \otimes\tilde{x}_{ij}^p \right]^{\frac{1}{p}}
    \end{array}
\end{equation}

\end{definition}

It is worth mentioning that, based exclusively on mathematical formalisation, we cannot define it either when the power tends to infinity or zero. However, we can make use of the fuzzy calculus to solve this issue \cite{bullen2013}. The following result shows us the existence of the $p$-mean for these particular cases.

\begin{theorem}[Score consistency]\label{theorem:existence}
    The score function $\tilde{\M}_i^p$ is well-defined for any $p\in\R\cup\{-\infty,+\infty\}$ whenever all the components of the decision-matrix are positive.
\end{theorem}
\begin{proof} 
    According to the hypothesis, the decision matrix $[\tilde{x}_{ij}]$ is $\succeq$-bounded by any positive fuzzy quantity. Then, we can guarantee all the axioms of the fuzzy arithmetic for each component of the LR-fuzzy numbers, and so the $p$-norm of the $\tilde{\M}_i^p$ function.
\end{proof}

According to Theorem~\ref{theorem:existence}, we can define the following particular cases for our $\tilde{\M}^p$ score:
\begin{equation}
    \begin{array}{rrrl}
         p = -\infty & : & \displaystyle\tilde{\M}^{-\infty}_i(\tilde{r}_{ij},\tilde{w}) & = \displaystyle \widetilde{\min_{j\in V}}\left\{\tilde{r}_{ij}\right\}, \vspace{1mm}\\
         
         p = 0 & : & \displaystyle\tilde{\M}^{0}_i(\tilde{r}_{ij},\tilde{w}) & = \displaystyle \bigotimes_{j\in V}\tilde{r}_{ij}^{\tilde{w}_j}, \vspace{1mm}\\
         
         p = +\infty & : & \displaystyle\tilde{\M}^{+\infty}_i(\tilde{r}_{ij},\tilde{w}) & = \displaystyle \widetilde{\max_{j\in V}}\left\{\tilde{r}_{ij}\right\}, \vspace{1mm}\\
    \end{array}
\end{equation}

In this manner, we can modify the value of $p$ depending on the decision-makers' demands on how to aggregate moral values. As shown, their outlook can vary from a totally pessimistic approach (i.e. $p = -\infty$) to an optimistic approach that values the best possible case (i.e. $p = +\infty$). An interesting result of this type of aggregation is that it satisfies the monotonicity property with respect to the $p$ power.

\begin{theorem}[Score monotony]\label{theorem:monotony}
    Given a decision matrix $\tilde{X}$ composed of positive LR-fuzzy numbers and $p,q\in\R$ so that $p\le q$, then
    \begin{equation}\notag
        \tilde{\M}_{i}^{-\infty} \lesssim
        \tilde{\M}_{i}^{p} \lesssim
        \tilde{\M}_{i}^{q} \lesssim
        \tilde{\M}_{i}^{+\infty}.
    \end{equation}
\end{theorem}

\begin{proof}
    According to LR-fuzzy arithmetic, we can proceed with the component-by-component calculations. Since we know that the monotonicity of real-valued power means holds for the non-negative case \cite{bullen2013}, then the statement also holds for each component of the decision matrix. In addition, for the non-negativity of the matrix elements $\tilde{x}_{ij}$, we can ensure a well-defined exponentiation. Hence, the component-wise operator is monotonic for LR-fuzzy numbers.
\end{proof}

From this theorem, we can conclude that the score function of Rankzzy is a monotonically increasing fuzzy operator whose image is limited by the minimum and the maximum elements of the decision matrix. This mathematical property is quite useful because we can fit the algorithm according to the demands and value system considered in the decision problem. For such a reason, these aggregations are widely used to solve hierarchical evaluation problems \cite{Guh2008}. Moreover, the stakeholders can know a priori the values between which the scores for each action will oscillate.

Another point to emphasise is that the $p$ parameter of the generalised mean can also be tuned. Depending on the needs, certain aggregations may be preferred over others. When computing the $\M^{p}$, we can perform this score considering the geometric interpretation of the Pythagorean means. If we want to analyse different variation rates or scales in the sample, cases of $p\in\{-1,0\}$ should be taken into account. In particular, both aggregations are convex, thus leading to better conditions when we carry out the optimisation problem. If an explainable and interpretable approach is preferred, then the choice should be made for the $p\in\{1,2\}$ cases.

\section{Rankzzy}\label{s:framework}

\begin{figure*}[h!]
	\centering
	\includegraphics[width=\textwidth]{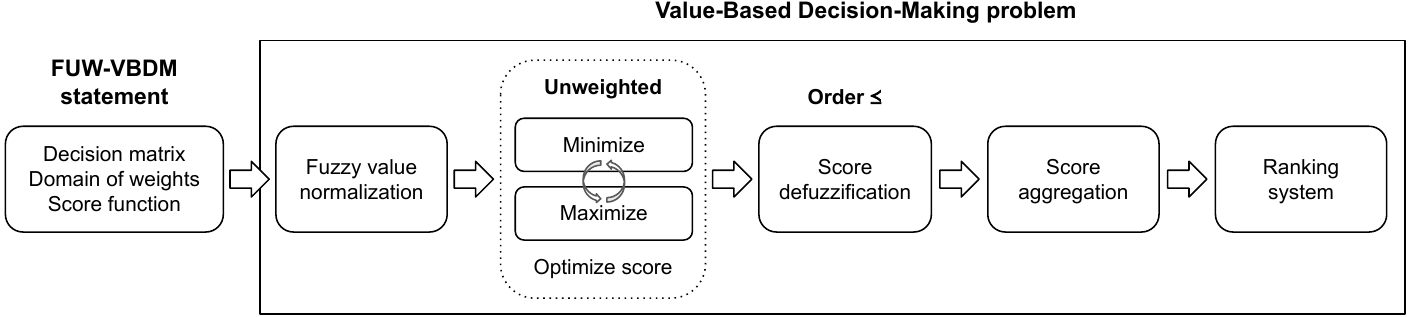}
	\caption{Flowchart with the implementation of the Rankzzy method to solve FUW-VBDM problems by producing customisable ranking systems.}
	\label{fig:rankzzy}
\end{figure*}

With the aim of combining the advantages of fuzzy logic and the non-arbitrary structure of the unweighted approaches in a MAS approach, we propose the Rankzzy method, a novel FUW-VBDM method that combines qualitative and quantitative variables simultaneously through a decision matrix. Rankzzy uses a data-driven method for ranking actions through the fuzzy $p$-mean score previously described. The set of values involved is divided as criteria to maximise ($V_{\max}$) or minimise ($V_{\min}$) so that $V = V_{\max} \cup V_{\min}$.

The justification for using both approaches can be explained from two perspectives. First, we can quantify the uncertainty attached to values using fuzzy sets. Second, the use of the unweighted approach allows us to reduce bias and use flexible weighting schemes instead of fixed ones. Therefore, we generate a framework that prevents some of the common drawbacks of decision-making techniques.

Although fuzzy operations are clearly stated, it is not an easy task to perform the fuzzy optimisation of such ranking functions \cite{Triantaphyllou2000, Zimmermann2001}. The main problem is that the set of fuzzy numbers is not a totally ordered group. Hence, it requires an order relationship $\preceq$ so that we can compare the fuzzy results for each alternative.

\subsection{Algorithmic Implementation}

In this paper, we propose a value-based method for supporting decision-makers when dealing with FUW-VBDM problems. We introduce a generalised fuzzy unweighted operator as a score function to rank actions based on their range of possible values within the domain of weights. In Fig~\ref{fig:rankzzy}, the flowchart of Rankzzy steps is graphically represented. A detailed algorithmic implementation is presented as follows.


\noindent\textbf{Input:} Fuzzy decision matrix $\tilde{X} = \left[\left(x^L_{ij}, x^R_{ij}, \alpha^L_{ij}, \alpha^R_{ij}\right)_{L_{ij}R_{ij}}\right]_{(i,j)\in A\times V}$, fuzzy lower $\tilde{l}$ and upper $\tilde{u}$ bounds of the $\tilde{\Omega}_{lu}$ set, power $p$ of the aggregation function, and stability parameter $\nu\in[0,1]$.

\begin{enumerate}[label=Step \arabic{*}, align=left]
    %
    %
    \item\label{item:normalisation_FuwMM} Normalize the fuzzy decision matrix for each action $i\in A$ and value $j\in V$ as:
    \begin{equation}\label{eq:normalisation}
    \left\{
        \begin{array}{rl}
            \tilde{r}_{ij} = 
            \dfrac{\tilde{x}_{ij} \ominus \tilde{x}_{j}^{\ominus} \oplus \tilde{\boldsymbol{\varepsilon}}}
            {\tilde{x}_{j}^{\oplus} \ominus \tilde{x}_{j}^{\ominus} \oplus \tilde{\boldsymbol{\varepsilon}}}
            &
            \hbox{if } j\in V_{\max},\\
            \tilde{r}_{ij} = 
            \dfrac{\tilde{x}_{j}^{\oplus} \ominus \tilde{x}_{ij} \oplus \tilde{\boldsymbol{\varepsilon}}}
            {\tilde{x}_{j}^{\oplus} \ominus \tilde{x}_{j}^{\ominus} \oplus \tilde{\boldsymbol{\varepsilon}}}
            &
            \hbox{if } j\in V_{\min},
        \end{array}
    \right.
    \end{equation}
    where
    $
    \displaystyle
        \tilde{x}_{j}^{\oplus} = \widetilde{\max_{i\in A}}\{\tilde{x}_{ij}\}
        \hbox{ and }
        \tilde{x}_{j}^{\ominus} = \widetilde{\min_{i\in A}}\{\tilde{x}_{ij}\}.
    $
    The $\tilde{\boldsymbol{\varepsilon}}$ stands for the fuzzy number $(\varepsilon, \varepsilon, 0, 0)$ of an arbitrary $\varepsilon>0$ small quantity to avoid zero division.
    
    %
    %
    
    %
    %
    \item\label{item:score_optimisation_FuwMM} Compute the mathematical constrained optimisation problem by considering $\tilde{\M}_i^p$ as objective function, for each action $i\in A$ as:
    \begin{eqnarray} \label{eq:FuwMM_optimisation}
    \displaystyle \tilde{\M}_i^{p-\min} = \min_{}\left\{\tilde{\M}^p_i(\tilde{r}_{ij},\tilde{w})\ |\ \tilde{w}\in\tilde{\Omega}_{lu}\right\},\\
    \displaystyle \tilde{\M}_i^{p-\max} = \max_{}\left\{\tilde{\M}^p_i(\tilde{r}_{ij},\tilde{w})\ |\ \tilde{w}\in\tilde{\Omega}_{lu}\right\},
    \end{eqnarray}
    where the fuzzy optimisation is conducted considering the $\preceq$ relationship.

    %
    %
    \item\label{item:defuzzification_FuwMM} Defuzzify computing the distance of each $\tilde{\M}_i^{p-\min}$ and $\tilde{\M}_i^{p-\max}$ using the vertex method with respect to the zero crisp number $\tilde{\mathbf{0}}$ to get the ${\M}_i^{p-\min}$ and ${\M}_i^{p-\max}$ scores for each action $i\in A$ as.

    %
    %
    \item\label{item:aggregation_FuwMM} From the stability parameter $0\le\nu\le1$, we aggregate the ${\M}_i^{p-\min}$ and ${\M}_i^{p-\max}$ scores as a convex $p$-mean combination, i.e.:
    \begin{equation} \label{eq:FuwMM_aggregated}
        {\M}_i = 
        \left[
        (1-\nu)\left({\M}_i^{p-\min}\right)^{p} + \nu\left({\M}_i^{p-\max}\right)^{p}
        \right]^{\frac{1}{p}},
    \end{equation}
    where the approach can be considered pessimistic if $\nu\le0.5$, optimistic if $\nu\ge0.5$ and balanced if $\nu=0.5$.
    
    %
    %
    \item\label{item:ranking_FuwMM} Rank alternatives in descending order regarding the ${\M}_i$ scores.
\end{enumerate}
\textbf{Output:} Vector of ${\M}_i$ scores, optimal weights $\left( \tilde{w}_i^{\min}, \tilde{w}_i^{\max} \right)$, and ranking system $\preccurlyeq_{\M}$ induced by the ${\M}_i$ scores.


As a result of the implementation of the Rankzzy algorithm, the output is composed of not only the $\M_i$ score vector, but also the optimal weights obtained in the optimisation problems of (\ref{eq:FuwMM_optimisation}) in \ref{item:score_optimisation_FuwMM}. In essence, they represent the maximum and maximum arguments as $ \tilde{w}_i^{\min} = \arg\min_{\tilde{w}\in\Omega_{lu}}\{\tilde{\M}^p_i(\tilde{r}_{ij},\tilde{w})\}$ and $\tilde{w}_i^{\max}  = \arg\max_{\tilde{w}\in\Omega_{lu}}\{\tilde{\M}^p_i(\tilde{r}_{ij},\tilde{w})\}$ per each action ${i\in A}$. This information helps us to understand the actions' assessments and also ensures the explainability and transparency of our method.

To solve the mathematical optimization problems presented in Step~\ref{item:score_optimisation_FuwMM}, it is essential to take into account that these problems constitute constrained non-linear formulations. Fuzzy weights are considered as decision variables over the $\tilde{\Omega}_{lu}$ set. The objective function consists of the crisp $p$-score through the vertex method. Hence, in this work, we have implemented a Sequential Least Squares Programming (SLSQP) algorithm for dealing with the latent quadratic behaviour of the generalised fuzzy means. Consistent with the original unweighted models, constrained-free methods are excluded, as the optimization problem is defined over a half-space.

\begin{example}
    For the selection of the exam method, we can perform the Rankzzy method using the decision matrix of Table~\ref{t:matrix_random_example} and the $\tilde{\Omega}_{lu}$ set of Table~\ref{t:bounds_random_example}. First, we have obtained the normalised decision matrix for a negligible value of $\varepsilon=10^{-4}$ as shown in Table~\ref{t:norm_matrix_random_example}. Second, we have arithmetically aggregated the $p$-scores, i.e. $p=1$. The results are displayed in Table~\ref{t:scores_random_example}.
    
    \begin{table*}[b!]
        \centering
        \caption{Normalised decision matrix (\ref{item:normalisation_FuwMM}) for the test selection problem.}
        \label{t:norm_matrix_random_example}
        \begin{tabularx}{\linewidth}{Xrr}
            \toprule
            Test &
            \multicolumn{1}{c}{Fairness} &
            \multicolumn{1}{c}{Costs} \\
            \midrule
            MC & $(\varepsilon, 0.128, 0.374, 0.602)$ &  $(0.609, 0.722, 0.869, 1.000)$ \\
            OA & $(0.464,  0.703,  0.948,  1.000)$    &  $(\varepsilon, 0.068, 0.155, 0.228)$ \\
            \bottomrule
        \end{tabularx}
    \end{table*}
    
    \begin{table*}[ht!]
    \centering
    \caption{Rankzzy scores using $p=1$ for the test selection problem broken down by the fuzzy and aggregated numbers. The last two rows contain the lower and upper optimal weights.}
    \label{t:scores_random_example}
    \begin{tabularx}{\linewidth}{XXXX}
        \toprule
        {Test} &
        \multicolumn{1}{c}{$\tilde{\M}_i^{1-\min}$} & 
        \multicolumn{1}{c}{$\tilde{\M}_i^{1-\max}$} & \\
        \midrule
        & \multicolumn{2}{c}{Fuzzy scores} & \\
        \cmidrule{2-4}
        MC &      $(0.061, 0.198, 0.473, 0.792)$ & $(0.091, 0.247, 0.553, 0.902)$ & \\ 
        OA &      $(0.279, 0.502, 0.790, 0.957)$ & $(0.348, 0.576, 0.892, 1.068)$ & \\
        \bottomrule
    \end{tabularx}
    
    \begin{tabularx}{\linewidth}{XXXX}
    \toprule
        Test
        &
        {${\M}_i^{1-\min}$} &
        {${\M}_i^{1-\max}$} &
        {${\M}_i$} \\
        \midrule
        & \multicolumn{2}{c}{Aggregated scores} & \\
        \cmidrule{2-3}
        MC &      $0.945$ &      $1.091$ &  $1.018$ \\
        OA &      $1.367$ &      $1.546$ &  $1.457$ \\
        \bottomrule
    \end{tabularx}
    \end{table*}

    \begin{table*}[ht!]
    \centering
    \caption{Weights of the Rankzzy algorithm using $p=1$ for the test selection problem broken down by the fuzzy lower and upper optimal weights.}
    \label{t:weights_random_example}
    \begin{tabularx}{\linewidth}{XXXX}
        \toprule
        Weights
        &
        Test
        &
        \multicolumn{1}{c}{Fairness} &
        \multicolumn{1}{c}{Costs} \\
        \midrule
        \multirow{2}{*}{$\tilde{W}^{\min}$}
        & MC &  (0.60, 0.70, 0.80, 0.90) &  (0.10,  0.15,  0.20,  0.25) \\
        & OA &  (0.60, 0.70, 0.80, 0.90) &  (0.10,  0.15,  0.20,  0.25) \\
        \midrule
        \multirow{2}{*}{$\tilde{W}^{\max}$}
        & MC &  (0.75, 0.80, 0.90, 1.00) &  (0.15,  0.20,  0.25,  0.30) \\
        & OA &  (0.75, 0.80, 0.90, 1.00) &  (0.15,  0.20,  0.25,  0.30) \\
        \bottomrule
    \end{tabularx}
    \end{table*}
    
    Then, according to the Rankzzy method, we can conclude by saying that the open answer test is preferred over the multiple choice option because $\M_{\text{MC}}\le\M_{\text{OA}}$. This is represented as $\text{MC}\preccurlyeq_{\M}\text{OA}$, thus indicating the preference of the open answer by our framework. Furthermore, it should be noted that ${\tilde{\M}}_{\text{MC}}^{1-\min} \lesssim {\tilde{\M}}_{\text{OA}}^{1-\min}$ and ${\tilde{\M}}_{\text{MC}}^{1-\max} \lesssim {\tilde{\M}}_{\text{OA}}^{1-\max}$, so this relationship is one of total domination. In Figure~\ref{fig:sensitivity_analysis_p_nu}, the sensitivity analysis of the $(p,\nu)$ input parameters is shown to reinforce the superiority of the open answer exam method. From the optimal weights $\tilde{W}^{\min}$ and $\tilde{W}^{\max}$ of the optimisation problem, we can see that the minimum score has attached the lower bound and the maximum the upper bound, just as expected, because $p=1$ aggregates in a linear manner.
    
    \begin{figure}[ht!]
      \centering
      \includegraphics[width=\linewidth]{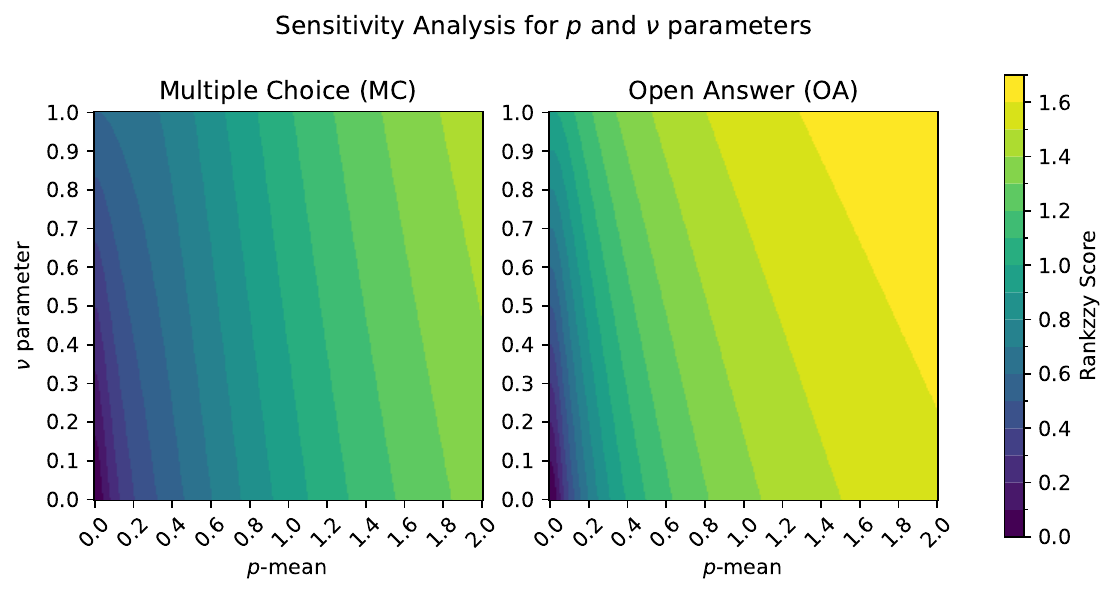}
      \caption{Sensitivity analysis of the $(p,\nu)$ parameters of the $\M$-score of the Rankzzy algorithm for a grid of $250\times250$ items considering $p\in[0,2]$ and $\nu\in[0,1]$.}
      \label{fig:sensitivity_analysis_p_nu}
    \end{figure}
\end{example}

\subsection{Properties of the Rankzzy algorithm}

The definition of the Rankzzy algorithm requires that certain mathematical properties be satisfied. These results are useful for agents because they guarantee the applicability of our method independently of the conditions set by the stakeholders.

\begin{theorem}[Strictly positive normalisation]\label{theorem:normalisation}
    Let $\tilde{X}$ be a fuzzy decision matrix. Then, for all $\varepsilon>0$, the normalised decision matrix obtained in \ref{item:normalisation_FuwMM} satisfies the following:
    \begin{equation}
         \fuzzyzero
         \prec
         \dfrac{\tilde{\boldsymbol{\varepsilon}}}{\tilde{x}_{j}^{\oplus} \ominus \tilde{x}_{j}^{\ominus} \oplus \tilde{\boldsymbol{\varepsilon}}}
        \preceq
        \tilde{r}_{ij}
        \preceq
        \fuzzyone.
    \end{equation}
\end{theorem}
\begin{proof} 
    According to Lemma~\ref{lemma:fuzzy_order_distance}, we know that $\preceq$ is more restrictive than $\lesssim$, so we can demonstrate the statement considering it. We just proved the $V_{\max}$ case since the other one is analogous. The fuzzy normalisation function used in \ref{item:normalisation_FuwMM} is defined as an affine transformation. In fact, we note that such a transformation is $\preceq$-increasing because $\tilde{x}_{j}^{\oplus} \ominus \tilde{x}_{j}^{\ominus} \oplus \tilde{\boldsymbol{\varepsilon}} \succeq \fuzzyzero$. Hence, the minimum and maximum values regarding $\preceq$ match with the image of $\tilde{x}_{j}^{\ominus}$ and $\tilde{x}_{j}^{\oplus}$ respectively. If we consider $\psi:\mathscr{X} \to \mathscr{X}$ as the fuzzy normalisation in (\ref{eq:normalisation}), we obtain the following:
    \begin{equation}
        \left\{
        \begin{array}{rl}
             \displaystyle\psi\left(\tilde{x}_{j}^{\ominus}\right) = &\dfrac{\tilde{\boldsymbol{\varepsilon}}}{\tilde{x}_{j}^{\oplus} \ominus \tilde{x}_{j}^{\ominus} \oplus \tilde{\boldsymbol{\varepsilon}}}, \vspace{1mm}\\
             \displaystyle\psi\left(\tilde{x}_{j}^{\oplus} \right) = & \fuzzyone.
        \end{array}
        \right.
    \end{equation}
\end{proof}

As a consequence of Theorems~\ref{theorem:existence}, \ref{theorem:monotony}, and \ref{theorem:normalisation}, we can state that Rankzzy is mathematically consistent for every FUW-VBDM scenario. Then, we reduce a complex and time-consuming task from the decision-makers' hands. Now they just have to consider the agents, actions, and values involved (\ref{s:decision_matrix}), determine the weighting space (\ref{s:domain_weights}) by consensus, and assess the type of approach to select the $p$ accordingly.



\section{Evaluation}\label{s:results}

The evaluation process of the Rankzzy algorithm for solving FUW-VBDM problems has been addressed from two perspectives. First, we have studied the execution time of our algorithm in real-sized and large value-based decision problems. 
Second, we have analysed the robustness of our framework when performing the ranking systems. To this end, we have evaluated it in terms of rank correlation.


\begin{figure}[ht!]
  \centering
  \includegraphics[width=\linewidth]{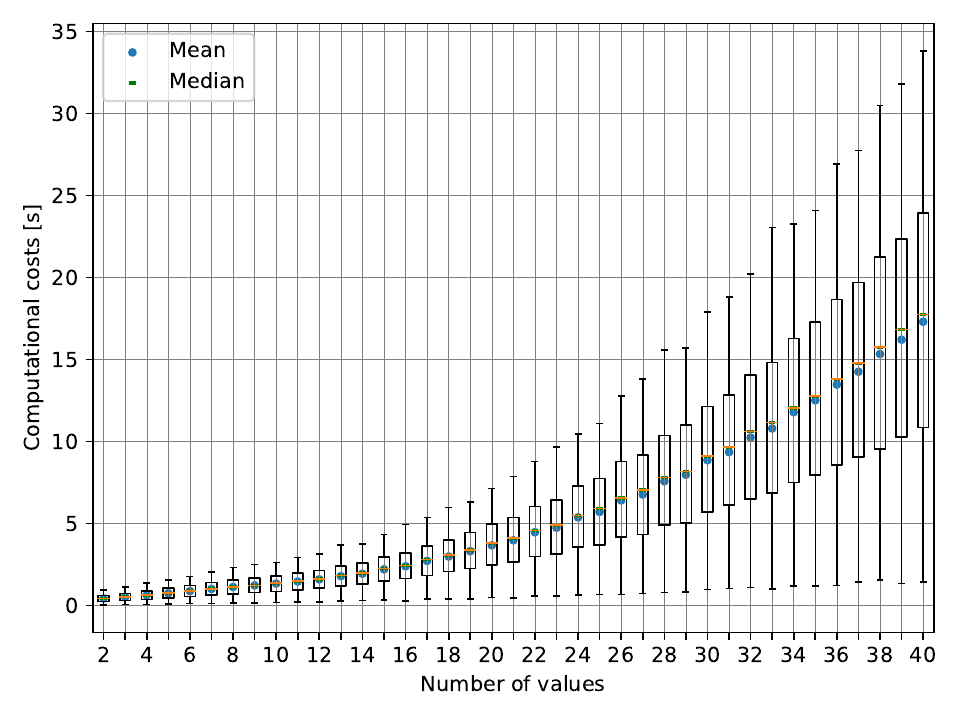}
  \caption{Computational costs attached to the implementation of Rankzzy over $50$ experiments enlarging the size of the decision matrix from $2\times 2$ to $50\times40$.}
  \label{fig:consistency}
\end{figure}

\begin{figure}[ht!]
  \centering
  \includegraphics[width=\linewidth]{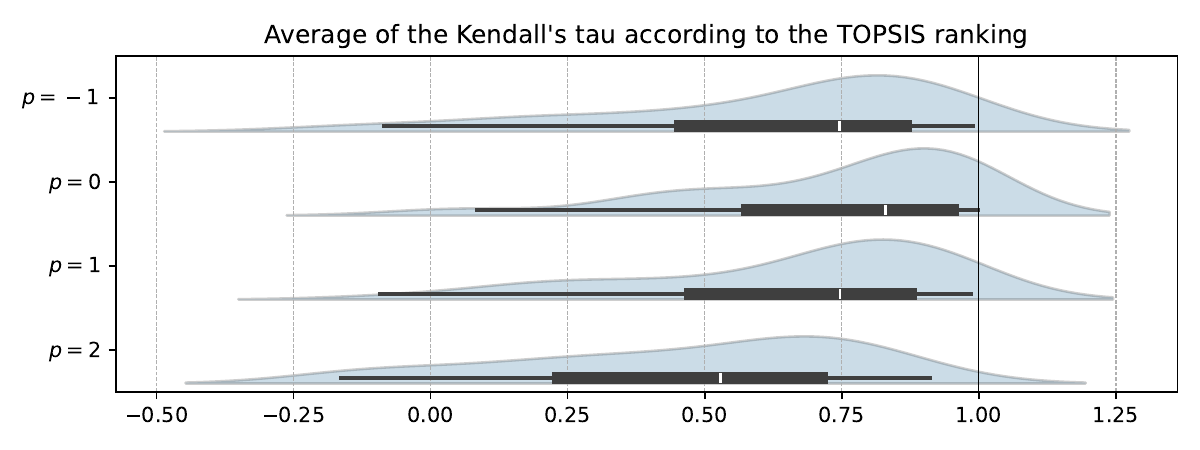}
  \caption{Distribution of the Kendall's $\tau$ correlation of the Rankzzy and TOPSIS ordinal placement considering $p\in\{-1,0,1,2\}$ over $50$ experiments.}
  \label{fig:kendall}
\end{figure}

\subsection{Computational costs}

We have run our algorithm in 50 randomised experiments considering the arithmetic aggregation in (\ref{eq:FWMM}). In each experiment, we set a $2\times 2$ fuzzy random decision matrix and then we calculate the time it takes to implement FUW-VBDM when increasing the matrix's size up to $50\times 40$. In Fig.~\ref{fig:consistency}, we have depicted the execution times of performing the ranking ($\preccurlyeq_{\M}$) associated with the number of values in the decision problem.

It should be noted that the application of our framework takes, on average, less than 5 seconds when we have at most 23 values. When between 24 and 31 values are present, the mean time remains below 10 seconds. In the worst-case scenario, a decision problem of $50\times 40$ size took 34.19 seconds to be solved.

\subsection{Rank correlation}

For this purpose, we have compared our algorithm's rankings with the TOPSIS technique ranking \cite{Hwang1981}. We have decided to use this decision-making algorithm because it is one of the SOTA techniques in MCDM \cite{Shih2022} and represents human choice's rationale through logical thinking \cite{Kaliszewski2016}. In this particular section, a direct comparison with other VBDM approaches has been deliberately avoided, as the primary objective is to assess the correlation of the complete ranking rather than the optimal solution obtained. In this experiment, we proceed with a $5\times 4$ random decision matrix ${X}$ with random optimal criteria $V_{\max}$ and $V_{\min}$. We subsequently fuzzify it by adding a random core and spread to get $\tilde{X}$. From the matrix $\tilde{X}$, we generate the noisy convex combination of the ideal solutions as follows:
\begin{equation}\label{eq:z}
    \tilde{z}_{\lambda} = 
    (1-\lambda)\odot\widetilde{\text{PIS}} \oplus \lambda\odot\widetilde{\text{NIS}}
    \oplus\lambda(1-\lambda)\odot\tilde{\epsilon},
\end{equation}
\begin{equation}\label{eq:pis_nis}
    \begin{array}{ccc}
        \widetilde{\text{PIS}}_{j} = 
        \begin{cases}
            \displaystyle\widetilde{\max_{i\in A}}\{\tilde{x}_{ij}\} & \hbox{if } j\in V_{\max}, \\
            \displaystyle\widetilde{\min_{i\in A}}\{\tilde{x}_{ij}\} & \hbox{if } j\in V_{\min},
        \end{cases}
        & 
        \hbox{and}
        &
        \widetilde{\text{NIS}}_{j} = 
        \begin{cases}
            \displaystyle\widetilde{\min_{i\in A}}\{\tilde{x}_{ij}\} & \hbox{if } j\in V_{\min}, \\
            \displaystyle\widetilde{\max_{i\in A}}\{\tilde{x}_{ij}\} & \hbox{if } j\in V_{\max},
        \end{cases}
    \end{array}
\end{equation}

so that $\lambda\in[0,1]$ is the parameter that generates the convex combination between the ideal solutions. The randomness is convexly added through the error term $\epsilon\sim\mathcal{N}(\bar{\tilde{x}},\sigma_{\tilde{x}})$, being $\bar{\tilde{x}}$ and $\sigma_{\tilde{x}}$ the mean and the standard deviation of the decision matrix respectively. The randomness is incorporated by employing the quadratic term $\lambda(1-\lambda)$ to ensure zero error when the $\tilde{z}_{\lambda}$ element approximates to $\widetilde{\text{PIS}}$ or $\widetilde{\text{NIS}}$ and higher error when the $\lambda$ term takes values close to $0.5$.

As classical TOPSIS is not designed to operate with fuzzy numbers, we applied the algorithm directly to $[X; z_{\lambda}]$, where $z_{\lambda}$ is the crisp version of (\ref{eq:z}) to $X$ and $[\cdot;\cdot]$ the matrix horizontal concatenation operator. For Rankzzy, we compute it to $[\tilde{X}; \tilde{z}_{\lambda}]$ so that it reflects the fuzziness of the problem. The rank correlation has been measured through Kendall's $\tau$ coefficient. In this manner, we can measure the degree of similarity between TOPSIS and Rankzzy rankings. If we consider $R$ and $T$ as the ordinal rankings of Rankzzy and TOPSIS methods, Kendall's score is computed as follows:

\begin{equation}
    \displaystyle\tau(R, T) =
    \frac{2}{N(N-1)}\sum_{i<j} \sign\left(R_{i} - R_{j}\right) \sign\big(T_{i} - T_{j}\big).
\end{equation}

Apart from studying the version utilised for the previous example, we have also considered the harmonic, geometric, and quadratic aggregations in (\ref{eq:FWMM}); i.e. we consider $p\in\{-1,0,1,2\}$. In this manner, we evaluate the impact of the Pythagorean means for our aggregation system.

In Fig.~\ref{fig:kendall} we have displayed the $\tau$ correlation distribution between the TOPSIS ranking and the output rankings grouped by the $p$ aggregations. Except for the quadratic mean, we have obtained a median of the $\tau$ rank correlation higher than 0.745, thus indicating a strong correlation regarding the TOPSIS technique. The explanation that the $p=2$ case of our model does not correlate as well as the others is due to the fact that TOPSIS uses the relative distances to the ideals, but $\M^{2}$ uses the fuzzy norm to the origin. For the $p=0$ case, the strong correlation is justified because the geometric mean indicates the central tendency of the sample values for each action.

\section{Conclusions}\label{s:conclusions}

As intelligent decision support systems are increasingly incorporated into critical domains, the challenge of aligning their outputs with human values becomes both urgent and unresolved. In this work, we introduce the Fuzzy Unweighted Value-Based Decision-Making framework for addressing VBDM problems. Our solution involves an integrated end-to-end MAS approach to cover the non-trivial steps of the problem and guide decision-makers in the process of selecting value-aligned actions. One novelty of our framework is the data-driven assembly of the decision matrix by simultaneously combining qualitative and quantitative criteria using LR-fuzzy numbers. This offers a more extensive methodology to decision-makers since they can explore more options when facing VBDM problems. Another novelty is the criteria assessment utilising fuzzy intervals instead of fixed weighting schemes. Thus, we address the subjectivity and the uncertainty attached to the MAS evaluation. Lastly, we incorporated an adaptable score function defined on the domain of weights for generating rankings of actions, thus meeting the user-specified demands. The selection of this score not only guarantees a correct aggregation of the values but also a mathematical consistency that eliminates any kind of indeterminacy on the part of our framework. For this, we have stated and proven two theorems that show these properties.

In order to provide a solution to the problem posed by the FUW-VBDM framework, we introduced the Rankzzy method. It makes use of the $\tilde{\M}^{p}$-score function to set the preferences among actions through the generated ranking, where decision-makers can modify the parameters to determine a particular approach based on their requirements. The algorithmic implementation of Rankzzy consists of five steps that facilitate the understanding and transparency of our method. The output ranking is driven by optimisation, both minimising and maximising, based on the fuzzy order relation. The output ranking is guided by the optimisation, both minimising and maximising, based on the $\preceq$ fuzzy order. Then, Rankzzy subsequently aggregates these scores on the same basis that $\tilde{\M}^{p}$ is generated, thus guaranteeing the same decision-making standard.

The applicability and usefulness of our framework have been shown in a running example to standardise exam methods in a university context with a detailed step-by-step implementation. With the FUW-VBDM perspective, we can effectively model the decision scenario and align the result with qualitative value-based criterion (fairness) and a quantitative resource-based attribute (cost). In Tables \ref{t:norm_matrix_random_example}, \ref{t:scores_random_example}, and \ref{t:weights_random_example}, we have displayed the numerical results to improve the readability and the follow-up of the calculations. As we indicated, the Rankzzy application can be fully customised to fit the decision problem.

For the evaluation stage, Rankzzy has proven effective in addressing value-based decision-making problems. It has not only shown a reasonable computational cost for dealing with FUW-VBDM problems but has also shown high effectiveness in solving them. First, we showed that decision problems stated with 40 criteria can be run in less than 18 seconds on average. Within the random sampling, the worst result in terms of computational cost was less than 35 seconds for a problem consisting of 50 actions and 40 attributes. Second, we compared the Rankzzy performance against TOPSIS in terms of rank correlation. To this end, we formulated a crisp problem and then we fuzzified it to get the expected fuzzy decision matrix of the FUW-VBDM statement. When considering the Pythagorean means for our $\tilde{\M}^{p}$-score, the results indicated a good performance of the output ranking, based on Kendall's Tau coefficient, despite the added uncertainty to the problem.

On the basis of the contributions and the results obtained, this work opens avenues for future investigation into the VBDM domain. From the reduced computational cost and the ranking performance of Rankzzy, future work should explore their applicability in high-dimensional case studies. Similarly, it would also be interesting to test its usefulness when applied in autonomous MAS with real-time problems.

\bibliographystyle{apalike}
\bibliography{references}

@article{Kaliszewski2016,
title = {Simple additive weighting—A metamodel for multiple criteria decision analysis methods},
journal = {Expert Systems with Applications},
volume = {54},
pages = {155-161},
year = {2016},
issn = {0957-4174},
doi = {https://doi.org/10.1016/j.eswa.2016.01.042},
url = {https://www.sciencedirect.com/science/article/pii/S0957417416000658},
author = {I. Kaliszewski and D. Podkopaev}
}

@book{Hwang1981,
    address = {Berlin, Heidelberg},
    author = {Hwang, Ching-Lai and Yoon, Kwangsun},
    doi = {10.1007/978-3-642-48318-9},
    isbn = {978-3-642-483},
    publisher = {Springer Berlin Heidelberg},
    series = {Lecture Notes in Economics and Mathematical Systems},
    title = {{Multiple Attribute Decision Making}},
    volume = {186},
    year = {1981}
}

@book{Triantaphyllou2000,
    author = {Triantaphyllou, Evangelos},
    year = {2000},
    title = {{Multi-Criteria Decision Making Methods: A Comparative Study}},
    booktitle = {Multi-Criteria Decision Making Methods: A Comparative Study},
    publisher = {Springer New York, NY},
    doi = {10.1007/978-1-4757-3157-6},
    isbn = {978-1-4419-4838-0},
    address   = {~},
    pages = {},
    volume = {44},
}

@book{Doumpos2013,
    author = {Doumpos, Michael and Grigoroudis, Evangelos},
    year = {2013},
    address   = {~},
    title = {Intelligent Decision Support Systems},
    booktitle = {Multicriteria Decision Aid and Artificial Intelligence},
    publisher = {John Wiley \& Sons, Ltd},
    doi = {10.1002/9781118522516},
    isbn = {9781119976394},
}

@book{Roy1996,
  title={Multicriteria methodology for decision aiding},
  author={Roy, Bernard},
  address = {~},
  volume={12},
  year={1996},
  doi={10.1007/978-1-4757-2500-1},
  isbn={978-1-4757-2500-1},
  issn={1571-568X},
  publisher={Springer Science \& Business Media}
}

@article{Liern2020,
    author = {Liern, V. and P{\'{e}}rez-Gladish, B.},
    doi = {10.1007/S10479-020-03718-1},
    issn = {1572-9338},
    number = {~},
    volume = {~},
    journal = {Annals of Operations Research 2020},
    pages = {1--23},
    publisher = {Springer},
    title = {{Multiple criteria ranking method based on functional proximity index: un-weighted TOPSIS}},
    year = {2020}
}

@article{Bouslah2023,
author = {Bouslah, K. and Liern, V. and Ouenniche, J. and P{\'e}rez-Gladish, B.},
title = {Ranking firms based on their financial and diversity performance using multiple-stage unweighted TOPSIS},
journal = {International Transactions in Operational Research},
volume = {30},
number = {5},
pages = {2485-2505},
doi = {https://doi.org/10.1111/itor.13143},
url = {https://onlinelibrary.wiley.com/doi/abs/10.1111/itor.13143},
eprint = {https://onlinelibrary.wiley.com/doi/pdf/10.1111/itor.13143},
year = {2023}
}

@Book{Shih2022,
    author="Shih, Hsu-Shih",
    title="TOPSIS Basics",
    bookTitle="TOPSIS and its Extensions: A Distance-Based MCDM Approach",
    year="2022",
    publisher="Springer International Publishing",
    address="Cham",
    pages="17--31",
    isbn="978-3-031-09577-1",
    doi="10.1007/978-3-031-09577-1_2",
}

@article{Kahraman2003,
title = {Fuzzy group decision-making for facility location selection},
journal = {Information Sciences},
volume = {157},
pages = {135-153},
year = {2003},
issn = {0020-0255},
doi = {10.1016/S0020-0255(03)00183-X},
author = {Cengiz Kahraman and Da Ruan and Ibrahim Doǧan},
}

@Inbook{Kahraman2006,
author="Kahraman, Cengiz
and G{\"u}lbay, Murat
and Kabak, {\"O}zg{\"u}r",
title="Applications of Fuzzy Sets in Industrial Engineering: A Topical Classification",
bookTitle="Fuzzy Applications in Industrial Engineering",
chapter="1",
year="2006",
publisher="Springer Berlin Heidelberg",
address="Berlin, Heidelberg",
pages="1--55",
isbn="978-3-540-33517-7",
doi="10.1007/3-540-33517-X_1",
}

@incollection{Reyna2009,
title = {Chapter 7 Development and Dual Processes in Moral Reasoning: A Fuzzy‐trace Theory Approach},
editor = {Brian H. Ross},
booktitle = {Psychology of Learning and Motivation},
address = {~},
publisher = {Academic Press},
volume = {50},
pages = {207-236},
year = {2009},
issn = {0079-7421},
doi = {10.1016/S0079-7421(08)00407-6},
author = {Valerie F. Reyna and Wanda Casillas}
}

@article{Zadeh1965,
    title = {Fuzzy sets},
    journal = {Information and Control},
    volume = {8},
    number = {3},
    pages = {338-353},
    year = {1965},
    issn = {0019-9958},
    doi = {10.1016/S0019-9958(65)90241-X},
    author = {L.A. Zadeh}
}

@article{Dubois1978,
author = {Dubois, Didier and Henri Prade},
title = {Operations on fuzzy numbers},
journal = {International Journal of Systems Science},
volume = {9},
number = {6},
pages = {613-626},
year  = {1978},
publisher = {Taylor & Francis},
doi = {10.1080/00207727808941724},
}

@book{Zimmermann2001,
  title={Fuzzy set theory and its applications},
  author={H.-J. Zimmermann},
  address   = {~},
  ISBN = {978-0-7923-7435-0},
  year={2001},
  doi = {10.1007/978-94-010-0646-0},
  publisher={Springer Dordrecht}
}

@book{deBarros2017,
author="de Barros, La{\'e}cio Carvalho
and Bassanezi, Rodney Carlos
and Lodwick, Weldon Alexander",
title="The Extension Principle of Zadeh and Fuzzy Numbers",
bookTitle="A First Course in Fuzzy Logic, Fuzzy Dynamical Systems, and Biomathematics: Theory and Applications",
year="2017",
publisher="Springer Berlin Heidelberg",
address="Berlin, Heidelberg",
pages="23--41",
isbn="978-3-662-53324-6",
doi="10.1007/978-3-662-53324-6_2",
url="https://doi.org/10.1007/978-3-662-53324-6_2"
}

@book{Bojadziev1996,
author = {Bojadziev, George and Bojadziev, Maria},
title = {Fuzzy Sets, Fuzzy Logic, Applications},
publisher = {WORLD SCIENTIFIC},
year = {1996},
doi = {10.1142/2867},
edition   = {},
address   = {~},
}

@book{Klir1988,
  title={Fuzzy Sets, Uncertainty, and Information},
  author={Klir, G.J. and Folger, T.A.},
  isbn={9780133459845},
  address   = {~},
  lccn={87006907},
  series={Prentice-Hall International editions},
  year={1988},
  publisher={Prentice Hall}
}

@article{Kahraman2015,
author = {Cengiz Kahraman and Sezi Cevik Onar and Basar Oztaysi},
title = {Fuzzy Multicriteria Decision-Making: A Literature Review},
journal = {International Journal of Computational Intelligence Systems},
volume = {8},
number = {4},
pages = {637-666},
year  = {2015},
publisher = {Taylor & Francis},
doi = {10.1080/18756891.2015.1046325},
}

@article{Xu2016,
title = {Information fusion for intuitionistic fuzzy decision making: An overview},
journal = {Information Fusion},
volume = {28},
pages = {10-23},
year = {2016},
issn = {1566-2535},
doi = {10.1016/j.inffus.2015.07.001},
author = {Zeshui Xu and Na Zhao}
}

@article{French1995,
author = {Simon French},
title = {Uncertainty and Imprecision: Modelling and Analysis},
journal = {Journal of the Operational Research Society},
volume = {46},
number = {1},
pages = {70-79},
year  = {1995},
publisher = {Taylor & Francis},
doi = {10.1057/jors.1995.8},
}

@article{Payzan2013,
title = {The Neural Representation of Unexpected Uncertainty during Value-Based Decision Making},
journal = {Neuron},
volume = {79},
number = {1},
pages = {191-201},
year = {2013},
issn = {0896-6273},
doi = {10.1016/j.neuron.2013.04.037},
author = {Elise Payzan-LeNestour and Simon Dunne and Peter Bossaerts and John P. O’Doherty},
}

@ARTICLE{Bach2011,
author = {Bach, Dominik R. and Hulme, Oliver and Penny, William D. and Dolan, Raymond J.},
title = {The known unknowns: Neural representation of second-order uncertainty, and ambiguity},
year = {2011},
journal = {Journal of Neuroscience},
volume = {31},
number = {13},
pages = {4811 – 4820},
doi = {10.1523/JNEUROSCI.1452-10.2011},
publication_stage = {Final},
source = {Scopus},
}

@article{Mardani2015,
title = {Fuzzy multiple criteria decision-making techniques and applications – Two decades review from 1994 to 2014},
journal = {Expert Systems with Applications},
volume = {42},
number = {8},
pages = {4126-4148},
year = {2015},
issn = {0957-4174},
doi = {10.1016/j.eswa.2015.01.003},
author = {Abbas Mardani and Ahmad Jusoh and Edmundas Kazimieras Zavadskas}
}

@article{Ramik1985,
title = {Inequality relation between fuzzy numbers and its use in fuzzy optimization},
journal = {Fuzzy Sets and Systems},
volume = {16},
number = {2},
pages = {123-138},
year = {1985},
issn = {0165-0114},
doi = {10.1016/S0165-0114(85)80013-0},
author = {Jaroslav Ram{\'i}k and Josef {\'i}m{\'a}nek}
}

@article{Chen2000,
title = {Extensions of the TOPSIS for group decision-making under fuzzy environment},
journal = {Fuzzy Sets and Systems},
volume = {114},
number = {1},
pages = {1-9},
year = {2000},
issn = {0165-0114},
doi = {10.1016/S0165-0114(97)00377-1},
author = {Chen-Tung Chen}
}

@article{Abbasbandy2009,
title = {A new approach for ranking of trapezoidal fuzzy numbers},
journal = {Computers \& Mathematics with Applications},
volume = {57},
number = {3},
pages = {413-419},
year = {2009},
issn = {0898-1221},
doi = {10.1016/j.camwa.2008.10.090},
author = {S. Abbasbandy and T. Hajjari}
}

@article{Guh2008,
title = {The fuzzy weighted average within a generalized means function},
journal = {Computers \& Mathematics with Applications},
volume = {55},
number = {12},
pages = {2699-2706},
year = {2008},
issn = {0898-1221},
doi = {10.1016/j.camwa.2007.09.009},
author = {Yuh-Yuan Guh and Rung-Wei Po and E. Stanley Lee}
}

@article{Opricovic2011,
title = {Fuzzy VIKOR with an application to water resources planning},
journal = {Expert Systems with Applications},
volume = {38},
number = {10},
pages = {12983-12990},
year = {2011},
issn = {0957-4174},
doi = {10.1016/j.eswa.2011.04.097},
author = {Serafim Opricovic}
}

@article{Fu2008,
title = {A fuzzy optimization method for multicriteria decision making: An application to reservoir flood control operation},
journal = {Expert Systems with Applications},
volume = {34},
number = {1},
pages = {145-149},
year = {2008},
issn = {0957-4174},
doi = {https://doi.org/10.1016/j.eswa.2006.08.021},
url = {https://www.sciencedirect.com/science/article/pii/S0957417406002636},
author = {Guangtao Fu}
}

@article{Jacquet1982,
title = {Assessing a set of additive utility functions for multicriteria decision-making, the UTA method},
journal = {European Journal of Operational Research},
volume = {10},
number = {2},
pages = {151-164},
year = {1982},
issn = {0377-2217},
doi = {10.1016/0377-2217(82)90155-2},
author = {E. Jacquet-Lagreze and J. Siskos}
}

@article{Watrobski2019,
title = {Generalised framework for multi-criteria method selection},
journal = {Omega},
volume = {86},
pages = {107-124},
year = {2019},
issn = {0305-0483},
doi = {10.1016/j.omega.2018.07.004},
author = {Jarosław Wątr{\'o}bski and Jarosław Jankowski and Paweł Ziemba and Artur Karczmarczyk and Magdalena Zioło}
}

@article{Ouenniche2018,
title = {An out-of-sample framework for TOPSIS-based classifiers with application in bankruptcy prediction},
journal = {Technological Forecasting and Social Change},
volume = {131},
pages = {111-116},
year = {2018},
issn = {0040-1625},
doi = {10.1016/j.techfore.2017.05.034},
author = {Jamal Ouenniche and Blanca P{\'e}rez-Gladish and Kais Bouslah}
}

@Article{Afsordegan2016,
author={Afsordegan, A.
and S{\'a}nchez, M.
and Agell, N.
and Zahedi, S.
and Cremades, L. V.},
title={Decision making under uncertainty using a qualitative TOPSIS method for selecting sustainable energy alternatives},
journal={International Journal of Environmental Science and Technology},
year={2016},
month={Jun},
day={01},
volume={13},
number={6},
pages={1419-1432},
issn={1735-2630},
doi={10.1007/s13762-016-0982-7},
}

@article{Pei2019,
title = {FLM-TOPSIS: The fuzzy linguistic multiset TOPSIS method and its application in linguistic decision making},
journal = {Information Fusion},
volume = {45},
pages = {266-281},
year = {2019},
issn = {1566-2535},
doi = {https://doi.org/10.1016/j.inffus.2018.01.013},
author = {Zheng Pei and Jing Liu and Fei Hao and Bin Zhou},
}

@article{REZAEI2015,
title = {Best-worst multi-criteria decision-making method},
journal = {Omega},
volume = {53},
pages = {49-57},
year = {2015},
issn = {0305-0483},
doi = {https://doi.org/10.1016/j.omega.2014.11.009},
url = {https://www.sciencedirect.com/science/article/pii/S0305048314001480},
author = {Jafar Rezaei},
}

@Article{LopezGarcia2023,
author={L{\'o}pez-Garc{\'i}a, A.
and Liern, V.
and P{\'e}rez-Gladish, B.},
title={Determining the underlying role of corporate sustainability criteria in a ranking problem using UW-TOPSIS},
journal={Annals of Operations Research},
year={2023},
month={Aug},
day={22},
issn={1572-9338},
doi={10.1007/s10479-023-05543-8},
url={https://doi.org/10.1007/s10479-023-05543-8}
}

@inproceedings{LopezGarcia2024,
author={Aar{\'o}n L{\'o}pez{-}Garc{\'i}a},
title={A Proposal for Selecting the Most Value-Aligned Preferences in Decision-Making Using Agreement Solutions},
booktitle={Proceedings of the 16th International Conference on Agents and Artificial Intelligence - Volume 1: EAA},
year={2024},
pages={461-470},
publisher={SciTePress},
organization={INSTICC},
doi={10.5220/0012586300003636},
isbn={978-989-758-680-4},
issn={2184-433X},
}

@article{HOSSEINIDEHSHIRI2022,
title = {A novel group BWM approach to evaluate the implementation criteria of blockchain technology in the automotive industry supply chain},
journal = {Expert Systems with Applications},
volume = {198},
pages = {116826},
year = {2022},
issn = {0957-4174},
doi = {https://doi.org/10.1016/j.eswa.2022.116826},
url = {https://www.sciencedirect.com/science/article/pii/S0957417422002822},
author = {Seyyed Jalaladdin {Hosseini Dehshiri} and Mir Seyed Mohammad Mohsen Emamat and Maghsoud Amiri}
}

@article{TAVANA2024,
title = {A novel fuzzy scenario-based stochastic general best-worst method},
journal = {Expert Systems with Applications},
volume = {252},
pages = {124246},
year = {2024},
issn = {0957-4174},
doi = {https://doi.org/10.1016/j.eswa.2024.124246},
url = {https://www.sciencedirect.com/science/article/pii/S0957417424011126},
author = {Madjid Tavana and Shahryar Sorooshian and Homa Rezaei and Hassan Mina}
}

@article{Hajiaghaei2023,
title = {Pythagorean Fuzzy TOPSIS Method for Green Supplier Selection in the Food Industry},
journal = {Expert Systems with Applications},
volume = {224},
pages = {120036},
year = {2023},
issn = {0957-4174},
doi = {https://doi.org/10.1016/j.eswa.2023.120036},
url = {https://www.sciencedirect.com/science/article/pii/S0957417423005389},
author = {Mostafa Hajiaghaei-Keshteli and Zeynep Cenk and Babek Erdebilli and Yavuz {Selim Özdemir} and Fatemeh Gholian-Jouybari}
}

@Article{Goldani2024,
author={Goldani, Nastaran
and Ishizaka, Alessio},
title={A hybrid fuzzy multi-criteria group decision-making method and its application to healthcare waste treatment technology selection},
journal={Annals of Operations Research},
year={2024},
month={May},
day={16},
issn={1572-9338},
doi={10.1007/s10479-024-06036-y},
url={https://doi.org/10.1007/s10479-024-06036-y}
}

@Article{Bas2024,
author={Bas, Mar{\'i}a C.
and Bol{\'o}s, Vicente J.
and Prieto, {\'A}lvaro E.
and Rodr{\'i}guez-Echeverr{\'i}a, Roberto
and S{\'a}nchez-Figueroa, Fernando},
title={A multi-criteria decision support system to evaluate the effectiveness of training courses on citizens' employability},
journal={Applied Intelligence},
year={2024},
month={Nov},
day={30},
volume={55},
number={1},
pages={57},
issn={1573-7497},
doi={10.1007/s10489-024-05967-0},
url={https://doi.org/10.1007/s10489-024-05967-0}
}

@inproceedings{Arias2024,
author = {Arias, Joaqu\'{\i}n and Moreno-Rebato, Mar and Rodriguez-Garcia, Jose Antonio and Ossowski, Sascha},
title = {Towards value-awareness in administrative processes: an approach based on constraint answer set programming},
year = {2024},
isbn = {9798400702433},
publisher = {Association for Computing Machinery},
address = {New York, NY, USA},
url = {https://doi.org/10.1145/3605098.3636022},
doi = {10.1145/3605098.3636022},
booktitle = {Proceedings of the 39th ACM/SIGAPP Symposium on Applied Computing},
pages = {770–778},
numpages = {9},
location = {Avila, Spain},
series = {SAC '24}
}

@inproceedings{Osman2024,
           pages = {1531 --1539},
       booktitle = {AAMAS '23: International Conference on Autonomous Agents and Multiagent Systems},
         address = {New York},
            year = {2024},
       publisher = {Association for Computing Machinery (ACM)},
          series = {AAMAS '24: Proceedings of the 23rd International Conference on Autonomous Agents and Multiagent Systems},
             doi = {10.5555/3635637.3663013},
           month = {May},
           title = {A Computational Framework of Human Values},
            isbn = {9798400704864},
          author = {Osman, Nardine and d'Inverno, Mark},
             url = {https://dl.acm.org/doi/10.5555/3635637.3663013}
}

@inproceedings{shen2025,
    title = "{V}alue{C}ompass: A Framework for Measuring Contextual Value Alignment Between Human and {LLM}s",
    author = "Shen, Hua  and
      Knearem, Tiffany  and
      Ghosh, Reshmi  and
      Yang, Yu-Ju  and
      Clark, Nicholas  and
      Mitra, Tanu  and
      Huang, Yun",
    editor = "Zhang, Chen  and
      Allaway, Emily  and
      Shen, Hua  and
      Miculicich, Lesly  and
      Li, Yinqiao  and
      M'hamdi, Meryem  and
      Limkonchotiwat, Peerat  and
      Bai, Richard He  and
      T.y.s.s., Santosh  and
      Han, Sophia Simeng  and
      Thapa, Surendrabikram  and
      Rim, Wiem Ben",
    booktitle = "Proceedings of the 9th Widening NLP Workshop",
    month = nov,
    year = "2025",
    address = "Suzhou, China",
    publisher = "Association for Computational Linguistics",
    url = "https://aclanthology.org/2025.winlp-main.15/",
    doi = "10.18653/v1/2025.winlp-main.15",
    pages = "75--86",
    ISBN = "979-8-89176-351-7"
}

@incollection{Roth1992,
title = {Chapter 16 Two-sided matching},
booktitle = {Handbook of Game Theory with Economic Applications},
publisher = {Elsevier},
volume = {1},
address   = {~},
pages = {485-541},
year = {1992},
issn = {1574-0005},
doi = {10.1016/S1574-0005(05)80019-0},
author = {Alvin E. Roth and Marilda Sotomayor}
}

@Inbook{Barbera2004,
    author="Barber{\`a}, Salvador
    and Bossert, Walter
    and Pattanaik, Prasanta K.",
    title="Ranking Sets of Objects",
    bookTitle="Handbook of Utility Theory: Volume 2 Extensions",
    chapter="4",
    year="2004",
    publisher="Springer US",
    address="Boston, MA",
    pages="893--977",
    isbn="978-1-4020-7964-1",
    doi="10.1007/978-1-4020-7964-1_4",
}

@InProceedings{Moretti2017,
author="Moretti, Stefano
and {\"O}zt{\"u}rk, Meltem",
editor="Rothe, J{\"o}rg",
title="Some Axiomatic and Algorithmic Perspectives on the Social Ranking Problem",
booktitle="Algorithmic Decision Theory",
year="2017",
publisher="Springer International Publishing",
address="Cham",
pages="166--181",
isbn="978-3-319-67504-6"
}

@Article{Ford1994,
author={Ford, Robert C.
and Richardson, Woodrow D.},
title={Ethical decision making: A review of the empirical literature},
journal={Journal of Business Ethics},
year={1994},
month={Mar},
day={01},
volume={13},
number={3},
pages={205-221},
issn={1573-0697},
doi={10.1007/BF02074820},
}

@incollection{wittmer2019,
  title={Ethical decision-making},
  address   = {~},
  author={Wittmer, Dennis P},
  booktitle={Handbook of administrative ethics},
  pages={481--507},
  year={2019},
  publisher={Routledge}
}

@article{Walstrom2006,
author = {Kent A. Walstrom},
title = {Social and Legal Impacts on Information Ethics Decision Making},
journal = {Journal of Computer Information Systems},
volume = {47},
number = {2},
pages = {1-8},
year  = {2006},
publisher = {Taylor & Francis},
doi = {10.1080/08874417.2007.11645948},
}

@article{Kleijnen2001,
title = {Ethical issues in modeling: Some reflections},
journal = {European Journal of Operational Research},
volume = {130},
number = {1},
pages = {223-230},
year = {2001},
issn = {0377-2217},
doi = {10.1016/S0377-2217(00)00024-2},
author = {J.P.C. Kleijnen}
}

@book{MacAskill2020,
    author = {MacAskill, William and Bykvist, Krister and Ord, Toby},
    title = "{Moral Uncertainty}",
    publisher = {Oxford University Press},
    year = {2020},
    month = {09},
    address   = {~},
    isbn = {9780198722274},
    doi = {10.1093/oso/9780198722274.001.0001},
    eprint = {https://academic.oup.com/book/31934/book-pdf/49946263/9780191033636\_web.pdf},
}

@ARTICLE{Astobiza2021,
  author={Astobiza, An{\'i}bal Monasterio and Toboso, Mario and Aparicio, Manuel and L{\'o}pez, Daniel},
  journal={IEEE Technology and Society Magazine}, 
  title={AI Ethics for Sustainable Development Goals}, 
  year={2021},
  volume={40},
  number={2},
  pages={66-71},
  doi={10.1109/MTS.2021.3056294}
}

@article{Beckstead2023,
author = {Beckstead, Nick and Thomas, Teruji},
title = {A paradox for tiny probabilities and enormous values},
journal = {Noûs},
year={2023},
volume = {n/a},
number = {n/a},
pages = {},
doi = {10.1111/nous.12462},
eprint = {https://onlinelibrary.wiley.com/doi/pdf/10.1111/nous.12462}
}

@article{Christen2012,
author = {Christen, Marius and Schmidt, Stephan},
title = {A Formal Framework for Conceptions of Sustainability – a Theoretical Contribution to the Discourse in Sustainable Development},
journal = {Sustainable Development},
volume = {20},
number = {6},
pages = {400-410},
doi = {10.1002/sd.518},
eprint = {https://onlinelibrary.wiley.com/doi/pdf/10.1002/sd.518},
year = {2012}
}

@article{Antunes2006,
title = {Participatory decision making for sustainable development—the use of mediated modelling techniques},
journal = {Land Use Policy},
volume = {23},
number = {1},
pages = {44-52},
year = {2006},
note = {Resolving Environmental Conflicts:Combining Participation and Muli-Criteria Analysis},
issn = {0264-8377},
doi = {10.1016/j.landusepol.2004.08.014},
author = {Paula Antunes and Rui Santos and Nuno Videira}
}

@inproceedings{murukannaiah2020,
author = {Murukannaiah, Pradeep K. and Ajmeri, Nirav and Jonker, Catholijn M. and Singh, Munindar P.},
title = {New Foundations of Ethical Multiagent Systems},
year = {2020},
isbn = {9781450375184},
publisher = {International Foundation for Autonomous Agents and Multiagent Systems},
address = {Richland, SC},
booktitle = {Proceedings of the 19th International Conference on Autonomous Agents and MultiAgent Systems},
pages = {1706–1710},
numpages = {5},
location = {Auckland, New Zealand},
series = {AAMAS '20}
}

@article{Mashayekhi2022,
author = {Mashayekhi, Mehdi and Ajmeri, Nirav and List, George F. and Singh, Munindar P.},
title = {Prosocial Norm Emergence in Multi-agent Systems},
year = {2022},
issue_date = {June 2022},
publisher = {Association for Computing Machinery},
address = {New York, NY, USA},
volume = {17},
number = {1–2},
issn = {1556-4665},
url = {https://doi.org/10.1145/3540202},
doi = {10.1145/3540202},
journal = {ACM Trans. Auton. Adapt. Syst.},
month = sep,
articleno = {3},
numpages = {24}
}

@inproceedings{Lera2022,
author = {Lera-Leri, Roger and Bistaffa, Filippo and Serramia, Marc and Lopez-Sanchez, Maite and Rodriguez-Aguilar, Juan},
title = {Towards Pluralistic Value Alignment: Aggregating Value Systems Through $l_{p}$-Regression},
year = {2022},
isbn = {9781450392136},
publisher = {International Foundation for Autonomous Agents and Multiagent Systems},
address = {Richland, SC},
booktitle = {Proceedings of the 21st International Conference on Autonomous Agents and Multiagent Systems},
pages = {780–788},
numpages = {9},
location = {Virtual Event, New Zealand},
doi = {10.5555/3535850.3535938},
series = {AAMAS '22}
}

@inproceedings{Ajmeri2020,
author = {Ajmeri, Nirav and Guo, Hui and Murukannaiah, Pradeep K. and Singh, Munindar P.},
title = {Elessar: Ethics in Norm-Aware Agents},
year = {2020},
isbn = {9781450375184},
publisher = {International Foundation for Autonomous Agents and Multiagent Systems},
address = {Richland, SC},
booktitle = {Proceedings of the 19th International Conference on Autonomous Agents and MultiAgent Systems},
pages = {16–24},
numpages = {9},
location = {Auckland, New Zealand},
series = {AAMAS '20}
}

@inproceedings{Yurrita2022,
  title={Towards a multi-stakeholder value-based assessment framework for algorithmic systems},
  author={Yurrita, Mireia and Murray-Rust, Dave and Balayn, Agathe and Bozzon, Alessandro},
  booktitle={Proceedings of the 2022 ACM Conference on Fairness, Accountability, and Transparency},
  doi = {10.1145/3531146.3533118},
  pages={535--563},
  year={2022}
}

@inproceedings{Lopez2017,
author = {Lopez-Sanchez, Maite and Serramia, Marc and Rodriguez-Aguilar, Juan A. and Morales, Javier and Wooldridge, Michael},
title = {Automating Decision Making to Help Establish Norm-Based Regulations},
year = {2017},
publisher = {International Foundation for Autonomous Agents and Multiagent Systems},
address = {Richland, SC},
booktitle = {Proceedings of the 16th Conference on Autonomous Agents and MultiAgent Systems},
pages = {1613–1615},
numpages = {3},
location = {S\~{a}o Paulo, Brazil},
series = {AAMAS '17}
}

@inproceedings{Serramia2020,
author = {Serramia, Marc and Lopez-Sanchez, Maite and Rodriguez-Aguilar, Juan A.},
title = {A Qualitative Approach to Composing Value-Aligned Norm Systems},
year = {2020},
isbn = {9781450375184},
publisher = {International Foundation for Autonomous Agents and Multiagent Systems},
address = {Richland, SC},
booktitle = {Proceedings of the 19th International Conference on Autonomous Agents and MultiAgent Systems},
pages = {1233–1241},
numpages = {9},
location = {Auckland, New Zealand},
series = {AAMAS '20}
}

@inproceedings{Serramia2018,
author = {Serramia, Marc and Lopez-Sanchez, Maite and Rodriguez-Aguilar, Juan A. and Rodriguez, Manel and Wooldridge, Michael and Morales, Javier and Ansotegui, Carlos},
title = {Moral Values in Norm Decision Making},
year = {2018},
publisher = {International Foundation for Autonomous Agents and Multiagent Systems},
address = {Richland, SC},
booktitle = {Proceedings of the 17th International Conference on Autonomous Agents and MultiAgent Systems},
pages = {1294–1302},
numpages = {9},
location = {Stockholm, Sweden},
series = {AAMAS '18}
}

@article{Serramia2023,
title = {Building rankings encompassing multiple criteria to support qualitative decision-making},
journal = {Information Sciences},
volume = {631},
pages = {288-304},
year = {2023},
issn = {0020-0255},
doi = {10.1016/j.ins.2023.02.063},
author = {Marc Serramia and Maite Lopez-Sanchez and Stefano Moretti and Juan A. Rodriguez-Aguilar}
}

@inproceedings{Serramia2024,
author = {Serramia, Marc and Lopez-Sanchez, Maite and Rodriguez-Aguilar, Juan A. and Moretti, Stefano},
title = {Value Alignment in Participatory Budgeting},
year = {2024},
isbn = {9798400704864},
publisher = {International Foundation for Autonomous Agents and Multiagent Systems},
address = {Richland, SC},
booktitle = {Proceedings of the 23rd International Conference on Autonomous Agents and Multiagent Systems},
pages = {1692–1700},
numpages = {9},
location = {Auckland, New Zealand},
series = {AAMAS '24}
}

@article{LeraLeri2024,
title = {Aggregating value systems for decision support},
journal = {Knowledge-Based Systems},
volume = {287},
pages = {111453},
year = {2024},
issn = {0950-7051},
doi = {https://doi.org/10.1016/j.knosys.2024.111453},
url = {https://www.sciencedirect.com/science/article/pii/S0950705124000881},
author = {Roger X. Lera-Leri and Enrico Liscio and Filippo Bistaffa and Catholijn M. Jonker and Maite Lopez-Sanchez and Pradeep K. Murukannaiah and Juan A. Rodriguez-Aguilar and Francisco Salas-Molina}
}

@book{bullen2013,
  title={Handbook of means and their inequalities},
  author={Bullen, Peter S},
  volume={560},
  year={2013},
  doi={0.1007/978-94-017-0399-4},
  publisher={Springer Science \& Business Media},
}

@phdthesis{mitesis2023,
  title={Evaluation of optimal solutions in multicriteria models for intelligent decision support},
  author={Aar{\'{o}}n L{\'{o}}pez-Garc{\'{i}}a},
  year={2023},
  school={Universitat de Val{\`e}ncia},
  note={Published on \url{https://roderic.uv.es/handle/10550/88785}},
}

@misc{IEEE_AS,
  author = {{IEEE Standards Association}},
  title = {Autonomous Systems},
  year = {2023},
  url = {https://standards.ieee.org/industry-connections/ec/autonomous-systems.html},
  urldate = {2024-07-10},
  organization = {IEEE},
  note = {IEEE Industry Connections: Autonomous Systems Initiative}
}

@misc{EU_AI_Guidelines,
  author = {{European Commission}},
  title = {Ethics Guidelines for Trustworthy {AI}},
  year = {2019},
  url = {https://digital-strategy.ec.europa.eu/en/library/ethics-guidelines-trustworthy-ai},
  urldate = {2024-07-10}, 
  organization = {European Commission},
  note = {Part of the EU's Digital Strategy}
}

@Article{LoPiano2020,
author={Lo Piano, Samuele},
title={Ethical principles in machine learning and artificial intelligence: cases from the field and possible ways forward},
journal={Humanities and Social Sciences Communications},
year={2020},
month={Jun},
day={17},
volume={7},
number={1},
pages={9},
issn={2662-9992},
doi={10.1057/s41599-020-0501-9},
url={https://doi.org/10.1057/s41599-020-0501-9}
}

\end{document}